\documentclass[11pt, a4paper]{lumia}

\usepackage[sort&compress]{natbib}
\usepackage{xspace}

\theoremstyle{plain}
\newtheorem{theorem}{Theorem}
\newtheorem{proposition}[theorem]{Proposition}
\newtheorem*{proposition*}{Proposition}

\theoremstyle{definition}

\theoremstyle{definition}

\def\eqref#1{equation~\ref{#1}}

\usepackage{graphicx}           
\usepackage{tikz}               
\usepackage[edges]{forest}      

\usepackage{url}                
\usepackage{xurl}               

\usepackage{array}              
\usepackage{longtable}          
\usepackage{multirow}           
\usepackage{makecell}           
\usepackage{ragged2e}           

\usepackage{mathtools}          
\usepackage{nicefrac}           

\usepackage{algorithm}          
\usepackage{algorithmicx}       
\usepackage{algpseudocode}      
\usepackage{listings}           

\usepackage{subcaption}         
\usepackage{wrapfig}            
\usepackage[export]{adjustbox}  

\usepackage{changes}            
\usepackage{xspace}             
\usepackage[normalem]{ulem}     
\usepackage{CJKutf8}            

\usepackage[tikz]{bclogo}       
\usepackage[framemethod=tikz]{mdframed} 

\usepackage{lipsum}             
\usepackage{tocloft}            
\usepackage{afterpage}          
\usepackage{bbding}             
\usepackage{epigraph}           
\usepackage{minitoc}            
\usepackage{multicol}           
\usepackage{textgreek}          

\newcolumntype{P}[1]{>{\RaggedRight\arraybackslash}p{#1}}

\definecolor{uclablue}{RGB}{39, 116, 174}
\definecolor{bigaired}{RGB}{156, 0, 0}
\definecolor{myblue}{HTML}{598BE7}
\definecolor{mildblue}{RGB}{31,119,180}
\definecolor{sectionblue}{RGB}{70, 130, 180}
\definecolor{methodblue}{RGB}{0, 150, 136}
\definecolor{bgblue}{RGB}{245,243,253}
\definecolor{ttblue}{RGB}{91,194,224}
\definecolor{mygreen}{rgb}{0.64, 0.56, 0.88}
\definecolor{myyellow}{rgb}{0.68, 0.6, 0.1}
\definecolor{fancygreen}{rgb}{0.33, 0.68, 0.20}
\definecolor{salmon}{rgb}{0.94, 0.52, 0.49}
\definecolor{tablegreen}{rgb}{0.82, 0.94, 0.75}
\definecolor{tableblue}{rgb}{0.81, 0.90, 0.94}
\definecolor{tablered}{rgb}{0.97, 0.85, 0.85}
\definecolor{tableorange}{rgb}{0.96, 0.85, 0.81}
\definecolor{myorange}{rgb}{1.0, 0.49, 0.0}
\definecolor{tlgreen}{rgb}{0.33, 0.68, 0.20}
\definecolor{darkgreen}{RGB}{0,100,0}
\definecolor{darkred}{RGB}{200, 0, 0}
\definecolor{customyellow}{HTML}{FFFACD}
\definecolor{refinegreen}{RGB}{0, 128, 75}
\definecolor{scoregreen}{RGB}{34, 139, 34}
\definecolor{hidden-blue}{RGB}{194,232,247}
\definecolor{hidden-black}{RGB}{20,68,106}
\definecolor{yes}{HTML}{C6EFCE}
\definecolor{no}{HTML}{FFC7CE}
\definecolor{partial}{HTML}{FFEB9C}
\definecolor{external}{HTML}{D9E1F2}
\definecolor{hdr}{HTML}{F2F2F2}
\definecolor{GRPOrow}{gray}{0.96}
\definecolor{FlowRLrow}{RGB}{225,236,255}
\definecolor{FlowBlue}{RGB}{80,120,210}
\definecolor{GRPOGray}{gray}{0.35}

\setlist[itemize]{leftmargin=20pt, noitemsep, topsep=0pt}

\NewDocumentCommand{\kaiyan}{mO{}}{\textcolor{purple}{\textsuperscript{\textit{kaiyan}}\textsf{\textbf{\small[#1]}}}}
\NewDocumentCommand{\yuxin}{mO{}}{\textcolor{cyan}{\textsuperscript{\textit{yuxin}}\textsf{\textbf{\small[#1]}}}}
\NewDocumentCommand{\bx}{mO{}}{\textcolor{green}{\textsuperscript{\textit{bx}}\textsf{\textbf{\small[#1]}}}}
\NewDocumentCommand{\at}{mO{}}{\textcolor{red}{\textsuperscript{\textit{AT}}\textsf{\textbf{\small[#1]}}}}
\NewDocumentCommand{\re}{mO{}}{\textcolor{blue}{\textsuperscript{\textit{RE}}\textsf{\textbf{\small[#1]}}}}
\NewDocumentCommand{\ybsun}{mO{}}{\textcolor{magenta}{\textsuperscript{\textit{youbang}}\textsf{\textbf{\small[#1]}}}}
\NewDocumentCommand{\runze}{mO{}}{\textcolor{orange}{\textsuperscript{\textit{runze}}\textsf{\textbf{\small[#1]}}}}
\NewDocumentCommand{\add}{mO{}}{\textcolor{darkgreen}{\textsuperscript{\textit{Maybe Consider Discuss}}\textsf{\textbf{[#1]}}}}

\newcommand{\cmark}{\textcolor{darkgreen}{\boldmath$\checkmark$}}
\newcommand{\xmark}{\textcolor{darkred}{\boldmath$\times$}}

\newenvironment{itemize*}%
 {\leftmargini=10pt\begin{itemize}%
  \setlength{\itemsep}{0pt}%
  \setlength{\parskip}{0pt}%
  }%
 {\end{itemize}}

\newenvironment{enumerate*}%
 {\begin{enumerate}%
  \setlength{\itemsep}{0pt}%
  \setlength{\parskip}{0pt}}%
 {\end{enumerate}}

\newcommand{\cellstatus}[1]{%
  \begingroup
  \StrTrim{#1}[\statusval]%
  \IfStrEq{\statusval}{Yes}{\cellcolor{yes}\cmark}{}%
  \IfStrEq{\statusval}{No}{\cellcolor{no}\xmark}{}%
  \IfBeginWith{\statusval}{Yes (}{\cellcolor{yes}\cmark~\textit{\statusval\unskip}}{}%
  \IfStrEq{\statusval}{Partial}{\cellcolor{partial}\textbf{Partial}}{}%
  \IfStrEq{\statusval}{External}{\cellcolor{external}\textbf{External}}{}%
  \endgroup
}

\newtcolorbox{myboxi}[1][]{
  breakable,
  title=#1,
  colback=red!5,
  colbacktitle=red!5,
  coltitle=black,
  fonttitle=\bfseries,
  bottomrule=0pt,
  toprule=0pt,
  leftrule=2pt,
  rightrule=2pt,
  titlerule=0pt,
  arc=0pt,
  outer arc=0pt,
  colframe=red,
}

\newtcolorbox{myboxnote}[1][]{
  breakable,
  title=#1,
  colback=orange!0,
  colbacktitle=orange!0,
  coltitle=black,
  fonttitle=\bfseries,
  bottomrule=0pt,
  toprule=0pt,
  leftrule=2pt,
  rightrule=2pt,
  titlerule=0pt,
  arc=0pt,
  outer arc=0pt,
  colframe=orange,
}

\newtcolorbox{myboxii}[1][]{
  breakable,
  freelance,
  title=#1,
  colback=white,
  colbacktitle=white,
  coltitle=black,
  fonttitle=\bfseries,
  bottomrule=0pt,
  boxrule=0pt,
  colframe=white,
  overlay unbroken and first={
  \draw[red!75!black,line width=3pt]
    ([xshift=5pt]frame.north west) -- 
    (frame.north west) -- 
    (frame.south west);
  \draw[red!75!black,line width=3pt]
    ([xshift=-5pt]frame.north east) -- 
    (frame.north east) -- 
    (frame.south east);
  },
  overlay unbroken app={
  \draw[red!75!black,line width=3pt,line cap=rect]
    (frame.south west) -- 
    ([xshift=5pt]frame.south west);
  \draw[red!75!black,line width=3pt,line cap=rect]
    (frame.south east) -- 
    ([xshift=-5pt]frame.south east);
  },
  overlay middle and last={
  \draw[red!75!black,line width=3pt]
    (frame.north west) -- 
    (frame.south west);
  \draw[red!75!black,line width=3pt]
    (frame.north east) -- 
    (frame.south east);
  },
  overlay last app={
  \draw[red!75!black,line width=3pt,line cap=rect]
    (frame.south west) --
    ([xshift=5pt]frame.south west);
  \draw[red!75!black,line width=3pt,line cap=rect]
    (frame.south east) --
    ([xshift=-5pt]frame.south east);
  },
}

\tcbset{
  takeawaysbox/.style={
    title=Takeaways,
    colback=lightblue!80,
    colframe=black,
    fonttitle=\bfseries\small,
    coltitle=white,
    colbacktitle=black,
    enhanced,
    attach boxed title to top left={xshift=2.5mm,yshift=-2.5mm},
    boxed title style={rounded corners, size=small, colframe=black, colback=black},
    width=\linewidth,
    arc=3.5mm
  }
}

\mdfdefinestyle{mystyle}{%
  rightline=true,
  innerleftmargin=10,
  innerrightmargin=10,
  outerlinewidth=3pt,
  topline=false,
  rightline=true,
  bottomline=false,
  skipabove=\topsep,
  skipbelow=\topsep
}

\tikzset{%
    every node/.style={font=\tiny},
    parent/.style =          {align=center,text width=2cm,rounded corners=3pt, line width=0.3mm, fill=gray!10,draw=gray!80},
    child/.style =           {align=center,text width=2.0cm,rounded corners=3pt, fill=blue!10,draw=blue!80,line width=0.3mm},
    grandchild/.style =      {align=center,text width=2cm,rounded corners=3pt},
    greatgrandchild/.style = {align=center,text width=1.5cm,rounded corners=3pt},
    greatgrandchild2/.style = {align=center,text width=1.5cm,rounded corners=3pt},    
    referenceblock/.style =  {align=center,text width=1.5cm,rounded corners=2pt},
    pretrain/.style =           {align=center,text width=2.0cm,rounded corners=3pt, fill=blue!10,draw=blue!80,line width=0.3mm},   
    pretrain_work/.style =           {align=center, text width=8.5cm,rounded corners=3pt, fill=blue!10,draw=blue!0,line width=0.3mm},  
    template/.style =           {align=center,text width=2.0cm,rounded corners=3pt, fill=red!10,draw=red!80,line width=0.3mm},   
    template_work/.style =           {align=center,text width=8.5cm,rounded corners=3pt, fill=red!10,draw=red!0,line width=0.3mm},    
    answer/.style =           {align=center,text width=2.0cm,rounded corners=3pt, fill= cyan!10,draw= cyan!80,line width=0.3mm},   
    answer_work/.style =           {align=center,text width=8.5cm,rounded corners=3pt, fill= cyan!10,draw= cyan!0,line width=0.3mm},      
    multiple/.style =           {align=center,text width=2.0cm,rounded corners=3pt, fill= orange!10,draw= orange!80,line width=0.3mm},   
    multiple_work/.style =           {align=center,text width=8.5cm,rounded corners=3pt, fill= orange!10,draw= orange!0,line width=0.3mm},        
    tuning/.style =           {align=center,text width=2.0cm,rounded corners=3pt, fill= magenta!10,draw= magenta!80,line width=0.3mm},   
    tuning_work/.style =           {align=center,text width=8.5cm,rounded corners=3pt, fill= magenta!10,draw= magenta!0,line width=0.3mm},          
}

\newcommand{\lstbg}[3][0pt]{{\fboxsep#1\colorbox{#2}{\strut #3}}}

\lstdefinelanguage{diff}{
  basicstyle=\ttfamily\small,
  morecomment=[f][\lstbg{red!20}]-,
  morecomment=[f][\lstbg{green!20}]+,
}

\lstdefinelanguage{diffpython}{
  language=diff,
  morekeywords={def, if, else, for, while, return, import, from, as, class, with, try, except, finally, raise, lambda, and, or, not, in, is, None, True, False},
  morecomment=[l]{\#},
  morestring=[b]",
  morestring=[b]',
}

\usepackage{xspace}

\usepackage{xcolor}

\definecolor{ForestGreen}{RGB}{34,139,34}
\definecolor{myyellow}{RGB}{181, 181, 27}

\usepackage{amsmath}
\usepackage{amssymb}
\usepackage{mathtools}
\usepackage{amsthm}
\usepackage{fontawesome5} 

\usepackage[capitalize,noabbrev]{cleveref}

\usepackage{multirow}
\usepackage{makecell}
\usepackage{enumerate}
\usepackage{enumitem}
\usepackage{pifont}

\definecolor{mygrey}{gray}{0.4}

\usepackage[ruled,algo2e,vlined]{algorithm2e}
\SetKwInOut{Input}{Input}\SetKwInOut{Output}{Output}
\SetKwComment{Comment}{$\triangleright$\ }{}

\definecolor{darkgreen}{RGB}{30, 130, 30}
\definecolor{cream}{RGB}{253, 250, 242}
\renewcommand{\cmark}{\textcolor{darkgreen}{\ding{51}}} 
\renewcommand{\xmark}{\textcolor{red}{\ding{55}}}       

\usepackage{listings}

\usepackage{ulem}
\usepackage{pifont}

\setheadertext{LUMIA Lab}
\definecolor{highlight}{gray}{0.9}

\correspondingemail{\emailicon\ yang\_luo@u.nus.edu \quad $^\dagger$ Corresponding Author}

\title{On-Policy Adversarial Flow Distillation for Autoregressive Video Generation}
\setheadertitle{On-Policy Adversarial Flow Distillation for Autoregressive Video Generation}

\author{%
  Yang Luo$^{1}$, Shengju Qian$^{2 \dagger}$,
  Xiaohang Tang$^{3}$, Zirui Zhu$^{1}$, Yong Liu$^{1}$, Xin Wang$^{2}$, Yang You$^{1}$\\
  $^1$National Univesity of Singapore
  $^2$LIGHTSPEED
  $^3$University College London \\
}

\begin{document}

\begin{abstract}
Autoregressive video generators are attractive for streaming, long-horizon, and interactive applications, but distilling strong black-box teachers into causal students remains difficult. The student must learn under its own rollout distribution, whereas practical teachers may expose only prompt-conditioned completed videos and may differ in architecture, capacity, temporal design, and sampling schedule. This interface makes supervised fine-tuning off-policy, score-based distillation inapplicable, and direct adversarial imitation too sparse for denoising-time credit assignment. We propose \emph{Adversarial Flow Distillation} (AFD), an on-policy framework for heterogeneous black-box video distillation. AFD queries the teacher and rolls out the current student on the same prompts, trains a prompt-paired Bradley--Terry discriminator to estimate clean-sample teacher--student discrepancy, and converts the resulting on-policy advantage into forward-process flow-matching updates on the student's own noised states. Thus, AFD provides dense velocity-field supervision while requiring no teacher scores, latents, denoising trajectories, step alignment, or reverse-chain reinforcement learning. Experiments across two causal AR student families show that AFD consistently improves motion- and physics-sensitive generation while preserving general video quality, and ablations validate the importance of adaptive on-policy feedback and forward-process credit assignment. The method requires only clean teacher videos and student rollouts, providing a practical route for distilling proprietary or heterogeneous video generators into efficient autoregressive students.
\end{abstract}

\maketitle


\section{Introduction}
Efficient deployment of modern video generators increasingly depends on distilling large diffusion, score-based, or flow-matching models into smaller students \citep{ho2020ddpm,song2021score,lipman2023flow,liu2022rectified,salimans2022progressive,song2023consistency,yin2024dmd,kim2025vip,chern2025livetalk}. In practice, many strong video teachers are accessible only as black-box samplers that return completed clips, without exposing scores, logits, latents, sampler states, or denoising trajectories \citep{brooks2024sora,polyak2024moviegen,klingteam2025klingomni}. This is especially restrictive for autoregressive (AR) students such as Self-Forcing \citep{huang2025selfforcing}: a black-box teacher may use a different architecture, temporal conditioning scheme, and long denoising schedule, while the deployable AR student generates causally with few denoising steps. Existing distillation recipes therefore lack an aligned supervision channel from the teacher's hidden generation process to the student's rollout distribution and intermediate noised states.

A straightforward adaptation strategy is supervised fine-tuning (SFT) on teacher-generated videos. Although SFT does not require teacher scores, it trains the AR student under teacher-induced prefixes rather than under the student's own rollout distribution. At inference time, the student conditions on its previously generated frames; local errors can therefore shift future denoising states and accumulate over the video horizon \citep{bengio2015scheduled}. A preliminary SFT sweep in Figure~\ref{fig:sft_trend} confirms this mismatch: longer off-policy training fails to consistently improve either VBench~\citep{huang2024vbench} or VideoPhy-2~\citep{bansal2026videophy2}.

The SFT mismatch motivates on-policy adaptation, yet existing on-policy objectives still assume supervision unavailable in black-box video distillation. DMD-style objectives require teacher scores, density ratios, or compatible noised states \citep{yin2024dmd}, none of which are available from a sampling-only video teacher. Reward- and preference-based diffusion alignment can optimize models from sample-level feedback \citep{black2023ddpo,wallace2024diffusiondpo,liu2025videoalign,liu2025flowgrpo,xue2025dancegrpo}, but a scalar score on a completed video does not identify which frames, motion patterns, or denoising states account for the teacher--student discrepancy. The useful signal is observable only at completed videos, while the object being trained is a time-dependent vector field evaluated on the student's own noised AR states. The key challenge is therefore not merely to assign a reward to a video, but to lift black-box video evidence into dense flow-matching supervision for the student's causal denoising process.

\begin{wrapfigure}{r}{0.50\textwidth}
\centering
\vspace{-1.0em}
\includegraphics[width=0.5\textwidth]{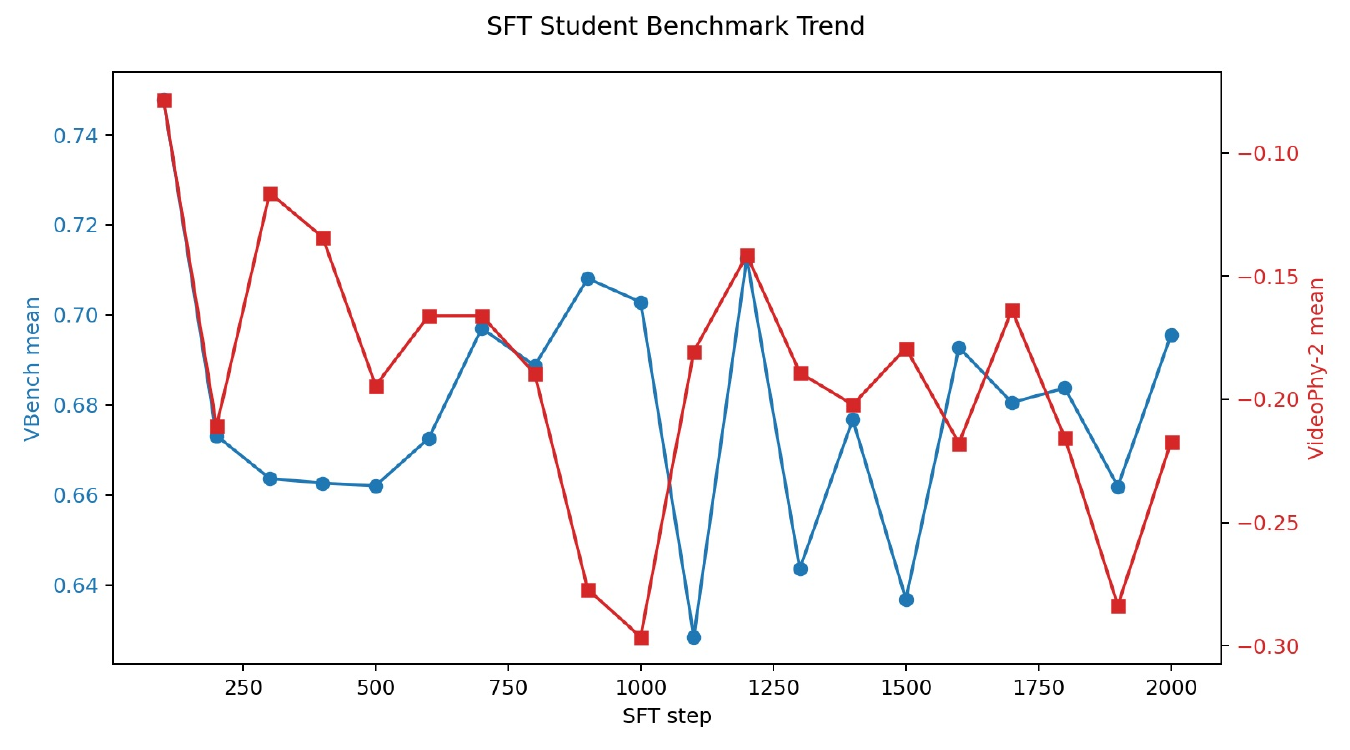}
\caption{Preliminary SFT trend for an AR video student. VBench and VideoPhy-2 scores decrease rather than improve monotonically as SFT steps increase.}
\label{fig:sft_trend}
\vspace{-1.0em}
\end{wrapfigure}

We propose \emph{Adversarial Flow Distillation} (AFD), an on-policy distillation framework for autoregressive video generation under heterogeneous black-box teacher access. AFD factorizes the problem into clean-sample distribution-ratio estimation and forward-process vector-field regression. For each prompt batch, it queries the teacher for completed videos, rolls out the current student under autoregressive self-conditioning, and trains a prompt-conditioned discriminator on teacher samples versus current student samples. The discriminator produces an on-policy advantage score, normalized against the current batch or prompt distribution, which is identifiable from black-box samples and co-evolves with the student.

AFD then uses this advantage inside a reward-weighted flow-matching objective rather than a reverse-trajectory policy-gradient objective. Student rollouts are forward-noised by the student's own schedule, producing the same kind of intermediate states on which its velocity field is trained. Within each on-policy batch, high-advantage rollouts form positive examples and low-advantage rollouts form negative examples; the student is then trained to move its vector field toward the forward velocity of higher-scoring rollouts and away from lower-scoring rollouts, with the discriminator advantage controlling the strength of this contrastive correction. This turns a black-box video-level signal into denoising-time updates on the student's own noised states. Consequently, AFD requires no teacher scores, no teacher--student step alignment, and no stored reverse trajectories, while still providing dense supervision beyond video-level adversarial training.
Our contributions are as follows.
\begin{itemize}
\item We identify black-box heterogeneous on-policy distillation as a core obstacle for autoregressive video students, and show that off-policy SFT, score-based DMD, and direct video-level adversarial training are mismatched to the limited teacher interface.
\item We introduce Adversarial Flow Distillation (AFD), a score-free distillation framework that estimates teacher--student discrepancy from completed videos and converts it into dense forward-process flow-matching updates on the student's own noised rollouts.
\item We evaluate AFD on two causal autoregressive video backbones, showing consistent gains on motion- and physics-sensitive metrics under black-box teacher access, together with ablations on domain adaptation and discriminator design.
\end{itemize}

\section{Related Work}
\textbf{Video diffusion and flow models.}
DDPM \citep{ho2020ddpm} and score-based SDEs \citep{song2021score} established denoising-time generation, while Flow Matching \citep{lipman2023flow} and Rectified Flow \citep{liu2022rectified} recast generation as continuous-time vector-field learning. DiT \citep{peebles2023dit} improved transformer diffusion scalability and now underlies video systems including Video Diffusion Models \citep{ho2022video}, Lumiere \citep{bartal2024lumiere}, CogVideoX \citep{yang2024cogvideox}, Movie Gen \citep{polyak2024moviegen}, Sora \citep{brooks2024sora}, Kling-Omni \citep{klingteam2025klingomni}, and Wan \citep{wanteam2025wan}. Recent video distillation work such as V.I.P. \citep{kim2025vip} and LiveTalk \citep{chern2025livetalk} studies online or on-policy recipes for efficient video generation. AFD focuses on a different transfer interface: an AR student learns on its own rollouts while a sampling-only teacher provides only completed videos.

\textbf{Black-box on-policy distillation.}
On-policy distillation \citep{lu2025opd, ye2025gad} keeps training on the student's own trajectory distribution while using a teacher for supervision. In language models, this often means teacher feedback on student-generated prefixes, and Rethinking OPD \citep{li2026rethinkingopd} analyzes when such feedback succeeds or fails. VLA-OPD \citep{zhong2026vlaopd} and Video-OPD \citep{li2026videoopd} apply the same on-policy idea to action and temporal grounding tasks. Our setting differs from LLM OPD because the teacher cannot provide token-level logits or local probability ratios: a sampling-only video teacher returns completed clips, while the student must learn a continuous-time video flow. AFD therefore estimates teacher--student distributional discrepancy with a discriminator and projects that signal to denoising-time states with DiffusionNFT.

\textbf{Adversarial and preference-guided alignment.}
DDPO \citep{black2023ddpo}, Diffusion-DPO \citep{wallace2024diffusiondpo}, and VideoAlign \citep{liu2025videoalign} show that learned or human feedback can guide diffusion models, but reverse-trajectory policy gradients are expensive for video. DiffusionNFT \citep{zheng2025diffusionnft} instead optimizes on the forward process from clean generated samples, avoiding likelihood estimation and reverse-trajectory storage; Astrolabe \citep{zhang2026astrolabe} adapts this view to distilled AR video alignment. Continuous Adversarial Flow Models \citep{lin2026cafm} further suggest that learned criteria can improve finite-capacity flow post-training. AFD uses these ideas for teacher distillation rather than generic reward maximization: feedback comes from a co-evolving teacher--student discriminator evaluated on on-policy student videos.

\vspace{-0.1em}
\section{Method}
\label{sec:method}
Let $\pi_T(x_0 | y)$ denote a black-box teacher that returns a video for prompt $y$, and let $\pi_\theta$ denote a causal autoregressive student flow model with velocity field $f_\theta(x_t,t,y)$. The teacher may differ from the student in architecture, capacity, latent representation, and sampling schedule. We assume \textit{no access to} teacher parameters, scores, latents, or denoising trajectories; the only shared interface is the completed prompt-conditioned video. At each iteration, prompts are sampled from a distribution $y \sim \mathcal{Y}$, the teacher returns $x_0^T \sim \pi_T(\cdot | y)$, and the current student produces an on-policy rollout $\hat{x}_0 \sim \pi_\theta(\cdot | y)$ under autoregressive self-conditioning. As shown in Figure~\ref{fig:method_opd}, AFD consists of an adaptive video discriminator that estimates teacher--student distributional discrepancy on student rollouts and a DiffusionNFT update that transfers this sample-level signal to the student's forward noising process.

\begin{figure}[t]
\centering
\includegraphics[width=\linewidth]{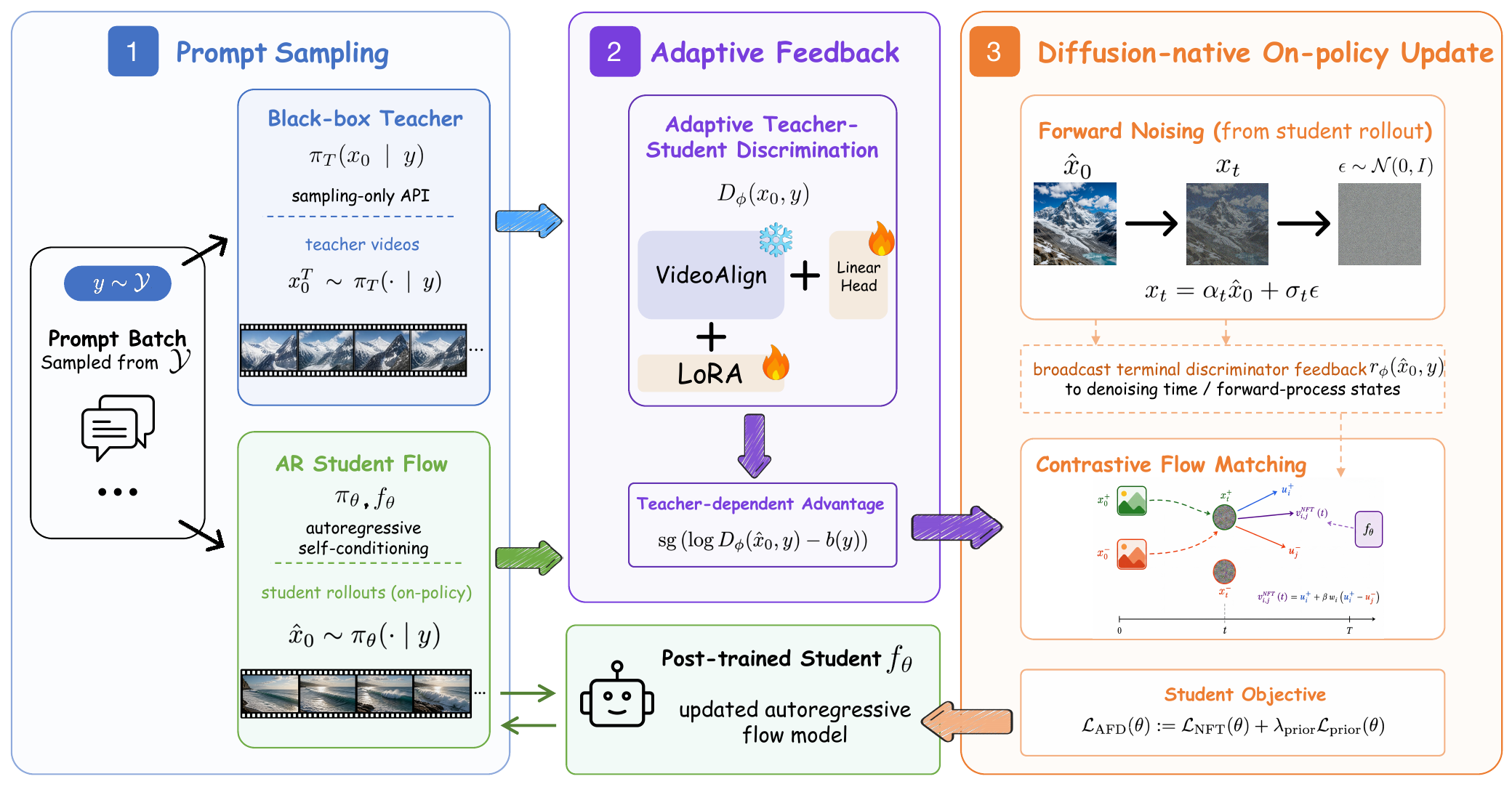}
\caption{Overview of AFD for black-box on-policy video distillation. A sampling-only teacher provides prompt-conditioned completed videos, while the current AR student provides on-policy rollouts. A co-evolving discriminator estimates teacher--student distributional discrepancy from completed samples, and DiffusionNFT converts this black-box sample-level signal into dense forward-process flow updates on the student's own noised states.}
\label{fig:method_opd}
\end{figure}

\subsection{Adaptive Teacher--Student Discrimination}
We train a prompt-conditioned spatiotemporal discriminator $D_\phi(x_0,y)$ that scores how likely a video $x_0$ is to be a teacher sample given prompt $y$. For each prompt, we treat the teacher sample as preferred over the current student rollout, yielding the Bradley–Terry (BT) loss:
\begin{equation}
\mathcal{L}_D(\phi)
:=
-\mathbb{E}_{(x_0^T,\hat{x}_0,y)}
\log \sigma\!\left(D_\phi(x_0^T,y)-D_\phi(\hat{x}_0,y)\right),
\end{equation}
where $\sigma(\cdot)$ is the logistic function. This pairwise objective is standard in preference modeling: it increases the discriminator margin between teacher samples and current student samples under the same prompt, without requiring calibrated absolute rewards. In practice, $D_\phi$ can be initialized from a video preference or text-video alignment model such as VideoAlign \citep{liu2025videoalign} and adapted with LoRA. Rather than defining a fixed global reward, the discriminator estimates the current discrepancy between teacher and student samples under the shared prompt distribution, providing on-the-fly feedback for on-policy rollouts when teacher scores are unavailable.

The discriminator induces a discrepancy signal, which we use as an adaptive reward function:
\begin{equation}
r_\phi(\hat{x}_0,y)=\operatorname{sg}\left(D_\phi(\hat{x}_0,y)-b(y)\right),
\end{equation}
where $\operatorname{sg}(\cdot)$ denotes stop-gradient and $b(y)$ is a batch or prompt-level baseline. Unlike a fixed reward model, this signal co-evolves with the student model and measures whether current student videos remain distinguishable from teacher videos under the current prompt distribution.

\subsection{Diffusion-Native On-Policy Update}
The discriminator signal is sample-level because it scores completed videos. Directly optimizing this signal with reverse-trajectory policy gradients would require treating the reverse denoising chain as an RL trajectory, which is costly for long, high-resolution videos. We instead adopt a forward-process diffusion optimization method that optimizes diffusion models using only clean samples and the forward noising process~\citep{choi2026rethinking}, namely DiffusionNFT \citep{zheng2025diffusionnft}. For a student rollout $\hat{x}_0$, timestep $t \in [0,1]$, and noise $\epsilon \sim \mathcal{N}(0,I)$, define the forward-noised sample and its corresponding flow-matching target as
\begin{equation}
\label{eq:forward_path}
x_t = \alpha_t \hat{x}_0 + \sigma_t \epsilon,\qquad
v=\dot{\alpha}_t \hat{x}_0+\dot{\sigma}_t\epsilon .
\end{equation}


We use the discriminator to score on-policy rollouts with $\pi_\theta$. We then determine the weights for positive and negative policy optimization. Concretely, for a minibatch of student videos $\{\hat{x}_{0,i}\}_{i=1}^B$, let $A_i=r_\phi(\hat{x}_{0,i},y_i)$ and normalize it to a weight $w_i \in [0,1]$ within the batch. We apply the forward process on this batch of videos following Equation \ref{eq:forward_path} to obtain noised samples $\{x_{t,i}\}_{i=1}^B$. Denoting the parametrized velocity as $v_\theta(x_{t,i},t,y_i)$, we define two operators for positive and negative policy optimization under the on-policy sampling setting:
\[
v_{\theta}^{+} := (1-\beta)\operatorname{sg}(v_\theta) + \beta v_\theta,
\qquad
v_{\theta}^{-} := (1+\beta)\operatorname{sg}(v_\theta) - \beta v_\theta .
\]
the negative-aware fine-tuning (NFT) loss whose optimum implicitly drives $v_\theta$ towards $v^+$ is
\begin{align}
\mathcal{L}_\text{NFT}(\theta)
=
\mathbb{E}_{\,t ,\,x_{t,i}}
\big[
\sum_{i=1}^B w_i \cdot \|v_{\theta}^{+}(x_{t,i},t,y_i) - v\|^{2}
+
(1-w_i) \cdot \|v_{\theta}^{-}(x_{t,i},t,y_i) - v\|^{2} \big],
\end{align}
Additionally, we include a regularization term in the objective function to avoid catastrophic forgetting $\mathcal{L}_{\mathrm{prior}} = \mathbb{E}_{\,t ,\,x_{t,i}}
\sum_{i=1}^B w_i \|v_{\theta}(x_{t,i},t,y_i) - v^\text{ref}(x_{t,i},t,y_i)\|^{2}$, where the reference model $v^\text{ref}$ is the initial pre-trained diffusion model. The objective of our method AFD is finally defined as:
\begin{equation}
\mathcal{L}_{\mathrm{AFD}}(\theta)
:= \mathcal{L}_{\mathrm{NFT}}(\theta)
+\lambda_{\mathrm{prior}}\mathcal{L}_{\mathrm{prior}}(\theta).
\end{equation}

\subsection{Black-Box Teacher Interface}

AFD is motivated by the measurability constraint imposed by a black-box video teacher. The teacher exposes only the clean-sample channel
\begin{equation}
\mathcal{O}_T:\quad y \mapsto x_0^T\sim \pi_T(\cdot| y),
\end{equation}
not $\nabla_x\log p_T(x_t| y)$, teacher latents, teacher reverse states, or a teacher transition kernel. Any admissible distillation signal must therefore be a function of completed teacher videos and completed student videos. This constraint is restrictive for an AR flow student, because the trained object is not a video classifier but a vector field $f_\theta(x_t,t,h_k,y)$ evaluated at noised states along student-induced histories $h_k=x_0^{(<k)}$.

Let $x_0=[x_0^{(1)},\ldots,x_0^{(K)}]$ denote a sequence of video blocks. The AR student induces
\begin{equation}
\pi_\theta(x_0 | y)
=
\prod_{k=1}^{K}
\pi_\theta\!\left(x_0^{(k)} | x_0^{(<k)},y\right),
\end{equation}
so student errors alter the future conditioning distribution. A bidirectional black-box teacher, however, is observed only through $\pi_T(x_0| y)$. Its hidden sampler may contain $M_T$ denoising evaluations, whereas the student may use $M_S\ll M_T$. Without teacher trajectories or a shared latent space, there is no observable alignment operator
\begin{equation}
\mathcal{A}_{m,s}: z_{m}^{T}\mapsto (h_k,x_s^{S})
\end{equation}
from a teacher step $m$ to a student AR denoising state $s$. The absence of such an operator clarifies the limitations of two common baselines. SFT trains under teacher-induced prefixes rather than the student's rollout distribution, so exposure bias accumulates along the video horizon \citep{bengio2015scheduled}. DMD-style objectives require either a teacher score or a teacher density ratio at the student's noised state,
\begin{equation}
\nabla_x \log p_T(x_t| y)
\quad\text{or}\quad
\log \frac{\pi_T(\hat{x}_0| y)}{\pi_\theta(\hat{x}_0| y)}.
\end{equation}
The score term is unobservable, and the clean-sample ratio can only be estimated from samples at $x_0$. The relevant objective is therefore not to reconstruct the teacher's hidden diffusion path, but to transfer completed-video evidence to the student's own noised states. DiffusionNFT matches this interface: it requires only clean samples from the current policy, a scalar score on those samples, and the student's known forward noising kernel.

This differs from GRPO-style diffusion RL. Methods such as Flow-GRPO \citep{liu2025flowgrpo} and DanceGRPO \citep{xue2025dancegrpo} make diffusion or flow policies amenable to online RL by casting the reverse denoising process as an MDP, introducing stochastic reverse trajectories, and optimizing group-relative policy objectives on sampled rollouts. This interface is well suited to reward-model alignment, where the optimized model's reverse sampler defines the environment. It is less suitable for black-box teacher distillation: the teacher provides no reverse actions, no teacher trajectory probabilities, and no step-level supervision on the student's reverse chain. Applying a reverse-MDP objective therefore introduces a synthetic credit-assignment layer that is not supported by the teacher interface. AFD instead keeps the teacher-dependent signal at the clean-sample level, where it is identifiable, and uses the forward noising kernel to induce the dense denoising-time structure.

\subsection{Forward-Process Credit Assignment}

\label{sec:theory}

AFD follows this design by separating teacher evidence extraction from denoising-time credit assignment. First, the discriminator estimates the density-ratio information identifiable from black-box clean samples. The optimal discriminator that minimizes the BT loss satisfies
\begin{equation}
\rho_\phi^*(\hat{x}_0,y)
:=
\log \frac{D_\phi^*(\hat{x}_0,y)}{1-D_\phi^*(\hat{x}_0,y)}
=
\log \frac{\pi_T(\hat{x}_0| y)}{\pi_\theta(\hat{x}_0| y)}.
\end{equation}
This ratio is defined on completed videos sampled from the current student distribution; it requires no teacher score, teacher timestep, or architectural alignment.

Second, AFD converts this clean-sample ratio into a vector-field target by applying the student's forward noising map. Let $\mathcal{F}_t$ denote this teacher-agnostic forward-noising operator. The key compatibility condition is that $\mathcal{F}_t$ is entirely student-side: it depends on the student's noising schedule and generated clean video, not on any teacher state. A generic sample-level advantage $r(\hat x_0,y)$ defines the tilted clean-video law following DiffusionNFT:
\begin{equation}
\pi^+(\hat x_0\mid y)
\propto
r(\hat x_0,y)
\pi_\theta(\hat x_0\mid y).
\end{equation}
Applying the student's forward noising kernel $p(x_t|\hat{x}_0)$ following Eq. \ref{eq:forward_path} to this tilted clean
distribution, it can induce a corresponding distribution over noisy states according to Bayes' Rule:
\begin{equation}
\pi^+_t(x_t|y)
\propto
\mathbb{E}_{\hat{x}_0\sim\pi_\theta(\cdot|y)}
\left[
r(\hat x_0,y)p(x_t|\hat{x}_0)
\right].
\end{equation}
Then the optimal positive flow-matching vector field for these noised marginals is the
conditional average forward velocity,
\begin{equation}
v^+(x_t,t,y)
=
\mathbb{E}_{\hat{x}_0\sim\pi^+(\cdot|y),\,p(x_t|\hat{x}_0)}
\left[
v
\,\middle|\,x_t,y
\right].
\end{equation}
where $v$ is the forward-path velocity in Eq.~\ref{eq:forward_path}. Thus, a sample-level reward over completed videos induces a dense
denoising-time vector-field target after being propagated through the
student's forward process.
This identity provides the forward-noising bridge: sample-level preferences define a denoising-time vector field after being propagated through $\mathcal{F}_t$. 
The black-box restriction and AFD therefore operate on compatible information structures: all teacher-dependent information is measurable from completed videos, while denoising-time structure is induced by the student's forward process. AFD converts completed-video teacher evidence into dense corrections on the student's own noised states without reconstructing the teacher's hidden trajectory or storing the student's reverse trajectory. Appendix~\ref{sec:appendix_theory} provides further theoretical derivation.


\subsection{Training Procedure}
\label{sec:training_procedure}

The optimization alternates between on-policy data collection, discriminator updates, and student velocity-field updates, as summarized in Algorithm \ref{alg:afd}. We implement the prior regularizer $\mathcal{L}_{\mathrm{prior}}(\theta)$ as a velocity-regression penalty against a frozen reference student to preserve the student's base capabilities.

\begin{algorithm}[H]
\caption{Adversarial Flow Distillation (AFD)}
\label{alg:afd}
\DontPrintSemicolon
\KwIn{Teacher API $\pi_T$, Student Flow Model $f_\theta$, Discriminator $D_\phi$, Prompt distribution $\mathcal{Y}$}
\KwOut{Post-trained Student Model $f_\theta$}
Initialize student parameters $\theta$ and EMA parameters $\bar{\theta} \leftarrow \theta$, and keep $f_{\mathrm{ref}}$ frozen\;
Initialize discriminator parameters $\phi$ (e.g., via LoRA on VideoAlign)\;

\While{not converged}{
    \tcp{1. On-Policy Data Collection}
    Sample batch of prompts $y \sim \mathcal{Y}$\;
    
    Query teacher for target videos: $x_0^T \sim \pi_T(\cdot | y)$\;
    
    Sample student rollouts: $\hat{x}_0 \sim \pi_\theta(\cdot | y)$\;
    
    \tcp{2. Adaptive Teacher-Student Discrimination}
    Compute discriminator loss $\mathcal{L}_D(\phi)$ comparing $x_0^T$ and $\hat{x}_0$\;
    
    Update discriminator $\phi \leftarrow \phi - \eta_D \nabla_\phi \mathcal{L}_D(\phi)$\;
    
    Compute baseline $b(y)$ (e.g., batch mean of $D_\phi(\hat{x}_0, y)$)\;
    
    Compute advantage scores $r_\phi(\hat{x}_0, y) = \operatorname{sg}(D_\phi(\hat{x}_0, y) - b(y))$\;
    
    \tcp{3. Diffusion-Native Student Update}
    Sample denoising timesteps $t \sim \mathcal{U}[0, 1]$ and noise $\epsilon \sim \mathcal{N}(0, I)$\;
    
    Construct forward-noised student states $x_t = \alpha_t \hat{x}_0 + \sigma_t \epsilon$\;
    Form positive/negative rollout sets from $r_\phi$ and evaluate $\mathcal{L}_{\mathrm{NFT}}(\theta)$ using $x_t$\;
    
    Evaluate prior regularization $\mathcal{L}_{\mathrm{prior}}(\theta)$ against reference student $f_{\mathrm{ref}}$\;
    
    Update student $\theta \leftarrow \theta - \eta_\theta \nabla_\theta \left( \mathcal{L}_{\mathrm{NFT}} + \lambda_{\mathrm{prior}} \mathcal{L}_{\mathrm{prior}} \right)$\;
    
    \tcp{4. EMA Update}
    Update target network $\bar{\theta} \leftarrow \beta \bar{\theta} + (1-\beta) \theta$\;
}
\end{algorithm}

\vspace{-0.3em}
\section{Experiments}
\label{sec:exp}

\subsection{Setup}
\label{sec:setup}

\textbf{Experimental setup.}
We evaluate AFD on two causal autoregressive student families, Self-Forcing and Causal-Forcing. The teacher is Seedance 2.0 \citep{bytedanceseed2026seedance}, accessed only as a prompt-conditioned video sampling API. The main experiment is a continual adaptation setting: a pretrained AR student is adapted to a physics-oriented target domain while preserving general video quality. We sample $200$ examples from the VideoPhy-2 physics benchmark and query the teacher on the same prompts to obtain black-box adaptation videos. Unless otherwise specified, discriminators are initialized from VideoAlign \citep{liu2025videoalign}, adapted with LoRA, and updated online against the current student. We report VBench \citep{huang2024vbench} dimensions grouped into Physics and General categories. Physics includes temporal flickering, motion smoothness, dynamic degree, human action, and spatial relationship; General is the mean over the remaining VBench dimensions. We also evaluate target-domain adaptation with VideoAlign Motion Quality (VideoAlign-MQ) \citep{liu2025videoalign} and VideoPhy-2 Physical Consistency (VideoPhy-2-PC) \citep{bansal2026videophy2}. Full hyperparameters are provided in Appendix~\ref{sec:hyperparameters}.

\textbf{Baselines.}
We consider four baselines:
\begin{itemize}
\item \textbf{Base}: the pretrained AR student before teacher adaptation.
\item \textbf{SFT}: supervised fine-tuning on teacher-generated videos without on-policy student rollouts.
\item \textbf{GAN}: adversarial video-level training with the teacher--student discriminator, excluding the forward-process policy update.
\item \textbf{Score-free DMD}: a DMD-style training scaffold with the score-based distribution-matching term removed, isolating sample-level supervision under the same black-box access constraints.
\end{itemize}
\subsection{Main Results}
\label{sec:main_results}

Tables~\ref{tab:main_results} and~\ref{tab:motion_physics_aux} show that AFD improves physics-sensitive generation under limited target-domain data while preserving general video generation capability. On Self-Forcing, AFD reaches the best Physics VBench Total ($87.55$), VideoAlign-MQ ($0.605$), and VideoPhy-2-PC ($4.20$), with a General Total close to the best baseline ($60.83$ vs.\ $60.95$). On Causal-Forcing, AFD gives the best Physics VBench Total ($88.52$) and VideoPhy-2-PC ($4.24$), ties the best VideoAlign-MQ ($0.661$), and keeps the General Total close to the highest score ($59.83$ vs.\ $60.32$). Compared with the adapted baselines, AFD gives the strongest dynamic degree on both student families, indicating better motion and physical behavior without a large drop in general quality.

\begin{table}[h]
\vspace{-10pt}
\centering
\caption{Main VBench results after black-box continual adaptation. Physics reports all dimensions and their average; General reports the group average.}
\label{tab:main_results}
\tiny
\setlength{\tabcolsep}{2.0pt}
\renewcommand{\arraystretch}{1.08}
\begin{adjustbox}{width=\textwidth}
\begin{tabular}{c@{\hspace{6pt}}c@{\hspace{4pt}}|@{\hspace{4pt}}cccccc|c}
\toprule
\multirow{2}{*}[-6pt]{\textbf{Model}} & \multirow{2}{*}[-6pt]{\textbf{Method}}
& \multicolumn{6}{c|}{\textbf{Physics (VBench)}}
& \textbf{General (VBench)} \\
\cmidrule(lr){3-8}\cmidrule(lr){9-9}
& 
& \textbf{Total$\uparrow$}
& \textbf{\makecell{Temporal\\Flick.}$\uparrow$}
& \textbf{\makecell{Motion\\Smooth.}$\uparrow$}
& \textbf{\makecell{Dynamic\\Degree}$\uparrow$}
& \textbf{\makecell{Human\\Action}$\uparrow$}
& \textbf{\makecell{Spatial\\Rel.}$\uparrow$}
& \textbf{Total$\uparrow$} \\
\midrule
\multirow{5}{*}{Self-Forcing}
& Base & 68.49 & 98.27 & 97.19 & 91.67 & 46.00 & 9.32 & 36.51 \\
& SFT & 58.69 & 96.64 & 97.40 & 44.44 & 34.00 & 20.97 & 37.62 \\
& GAN & 83.83 & 99.28 & 98.58 & 61.11 & 79.00 & 81.18 & 60.95 \\
& DMD & 81.88 & 99.01 & 98.95 & 52.78 & 81.00 & 77.66 & 59.98 \\
& \textbf{AFD} & \cellcolor{highlight}\textbf{87.55} & \cellcolor{highlight}98.78 & \cellcolor{highlight}98.30 & \cellcolor{highlight}79.17 & \cellcolor{highlight}80.00 & \cellcolor{highlight}81.50 & \cellcolor{highlight}60.83 \\
\midrule
\multirow{5}{*}{Causal-Forcing}
& Base & 86.76 & 98.20 & 97.90 & 87.50 & 76.00 & 74.20 & 59.03 \\
& SFT & 76.24 & 99.31 & 99.05 & 29.17 & 78.00 & 75.69 & 59.40 \\
& GAN & 83.07 & 98.58 & 98.08 & 66.67 & 76.00 & 76.00 & 59.44 \\
& DMD & 82.52 & 99.16 & 98.75 & 58.33 & 77.00 & 79.35 & 60.32 \\
& \textbf{AFD} & \cellcolor{highlight}\textbf{88.52} & \cellcolor{highlight}98.41 & \cellcolor{highlight}97.89 & \cellcolor{highlight}88.89 & \cellcolor{highlight}80.00 & \cellcolor{highlight}77.39 & \cellcolor{highlight}59.83 \\
\bottomrule
\end{tabular}
\end{adjustbox}
\end{table}

\begin{table}[t]
\centering
\caption{Physics evaluations after small-sample physics-domain adaptation.}
\label{tab:motion_physics_aux}
\footnotesize
\setlength{\tabcolsep}{8.0pt}
\renewcommand{\arraystretch}{1.10}
\begin{tabular}{llcc}
\toprule
\textbf{Model} & \textbf{Method}
& \textbf{\makecell{Motion\\Quality$\uparrow$}}
& \textbf{\makecell{Physical\\Consistency$\uparrow$}} \\
\midrule
\multirow{5}{*}{Self-Forcing}
& Base & 0.341 & 4.04 \\
& SFT & 0.296 & 3.72 \\
& GAN & 0.541 & 4.10 \\
& DMD & 0.420 & 3.72 \\
& \textbf{AFD} & \cellcolor{highlight}\textbf{0.605} & \cellcolor{highlight}\textbf{4.20} \\
\midrule
\multirow{5}{*}{Causal-Forcing}
& Base & 0.520 & 4.16 \\
& SFT & 0.499 & 4.17 \\
& GAN & 0.582 & 4.16 \\
& DMD & 0.661 & 4.14 \\
& \textbf{AFD} & \cellcolor{highlight}\textbf{0.661} & \cellcolor{highlight}\textbf{4.24} \\
\bottomrule
\end{tabular}
\end{table}

\begin{figure}[t]
\centering
\begin{minipage}[t]{0.49\textwidth}
\centering
\includegraphics[width=\linewidth]{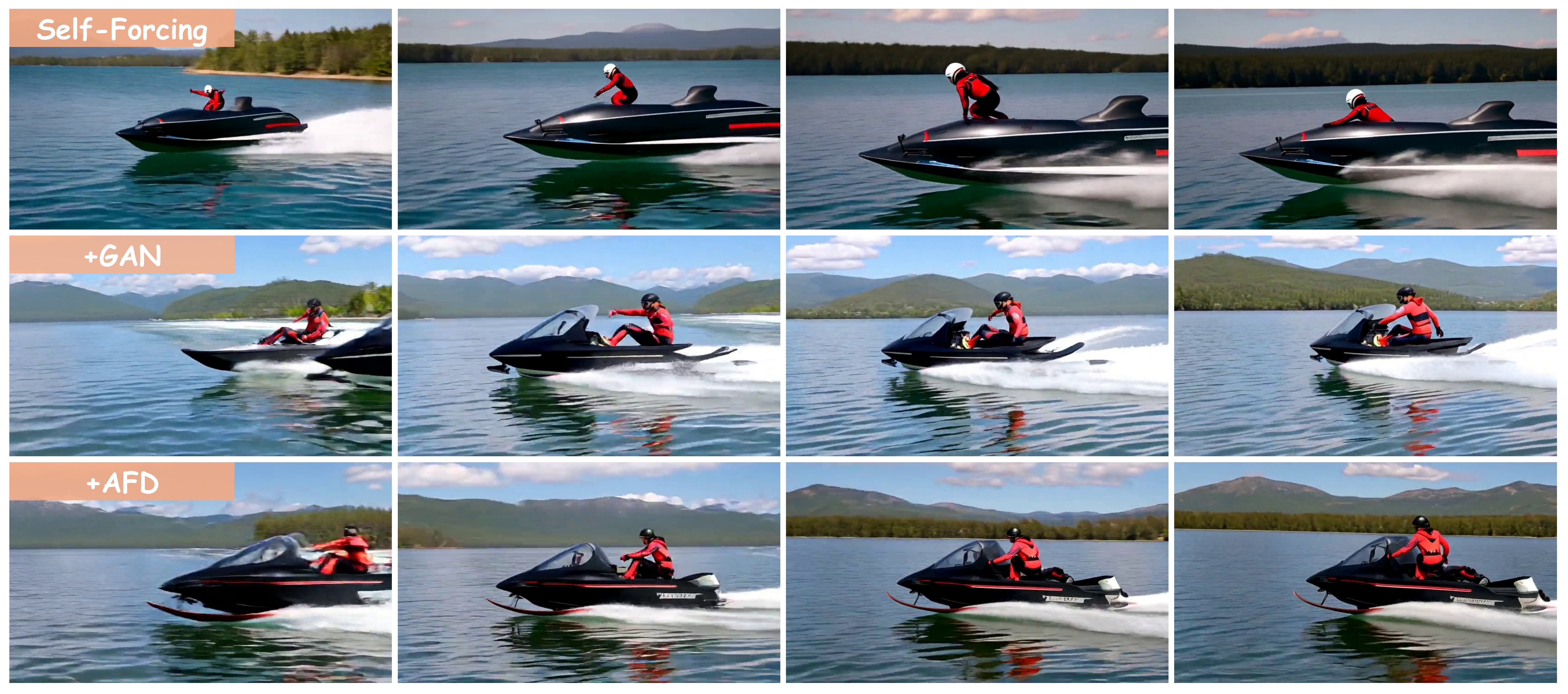}
\end{minipage}
\hfill
\begin{minipage}[t]{0.49\textwidth}
\centering
\includegraphics[width=\linewidth]{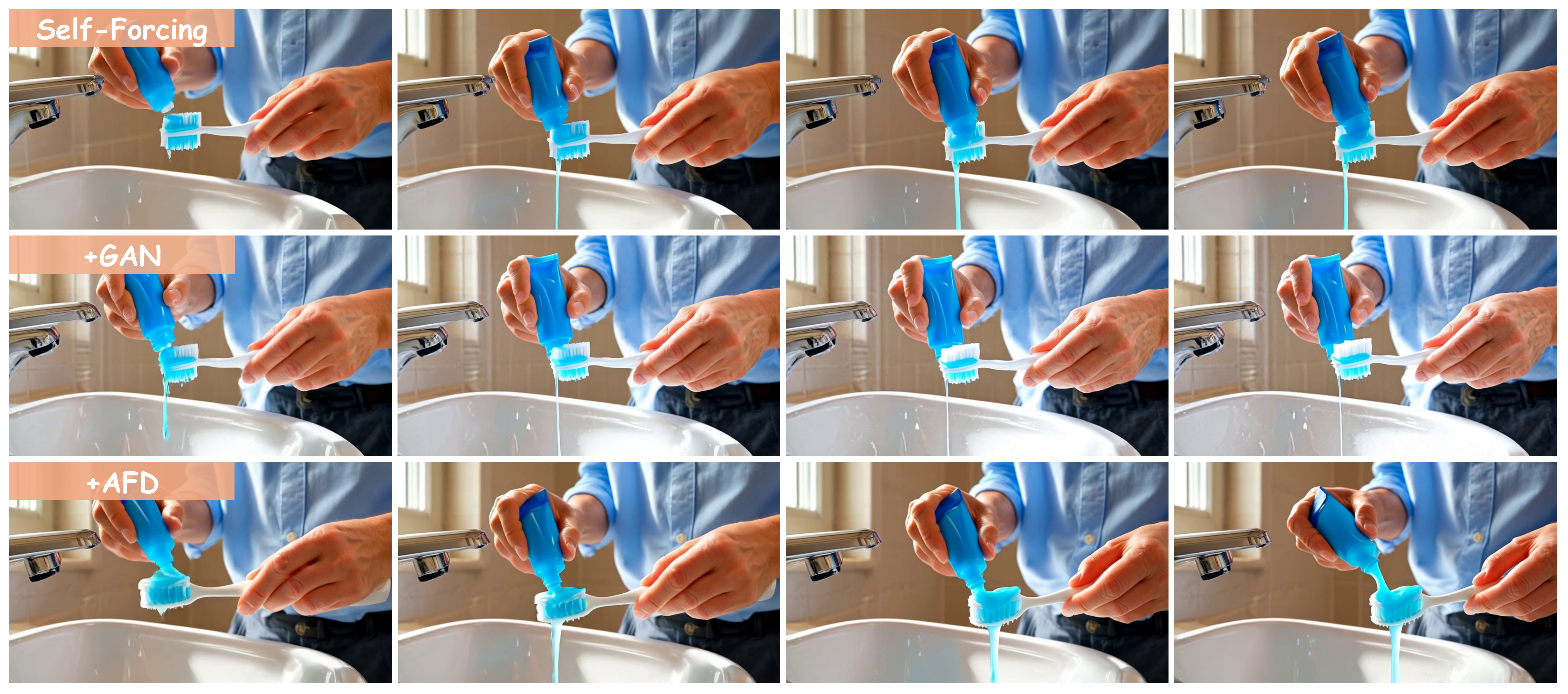}
\end{minipage}
\par\vspace{0.4em}
\begin{minipage}[t]{0.49\textwidth}
\centering
\includegraphics[width=\linewidth]{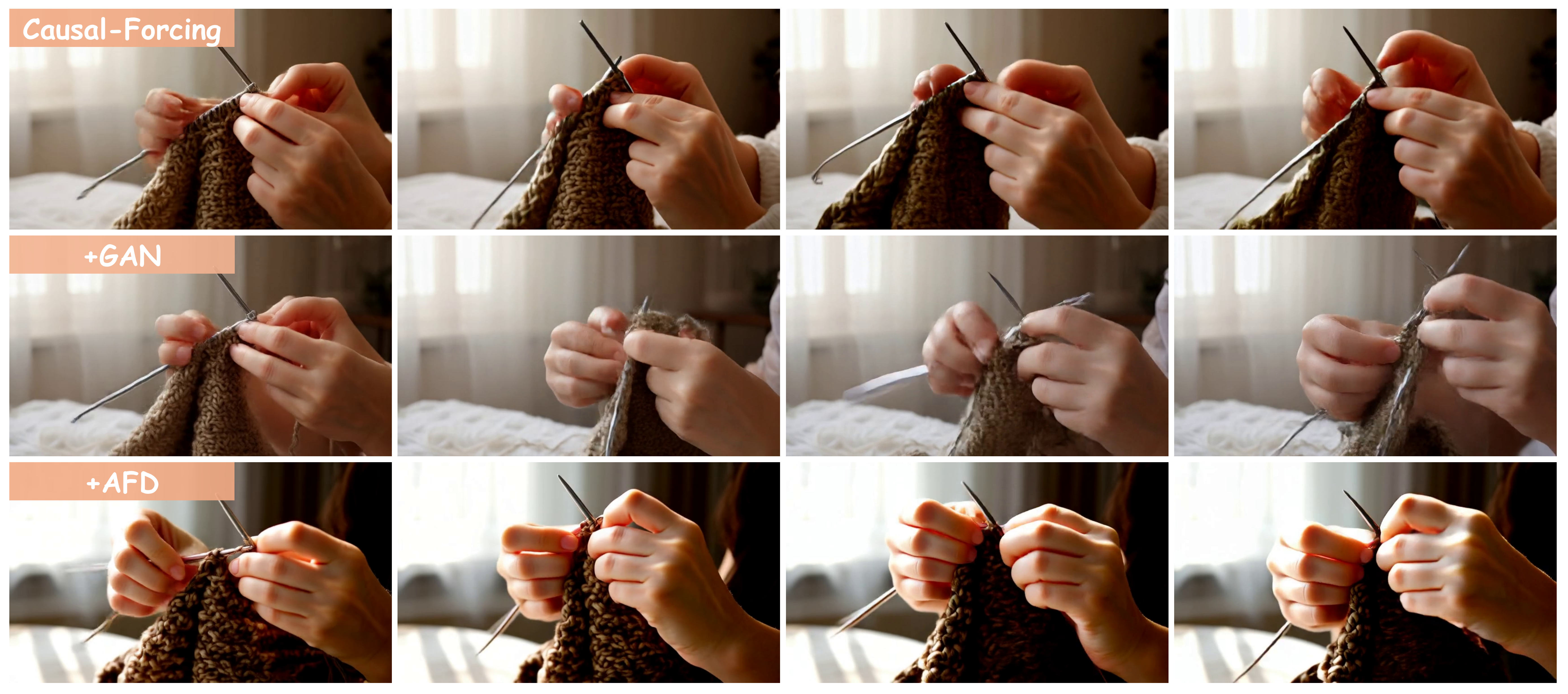}
\end{minipage}
\hfill
\begin{minipage}[t]{0.49\textwidth}
\centering
\includegraphics[width=\linewidth]{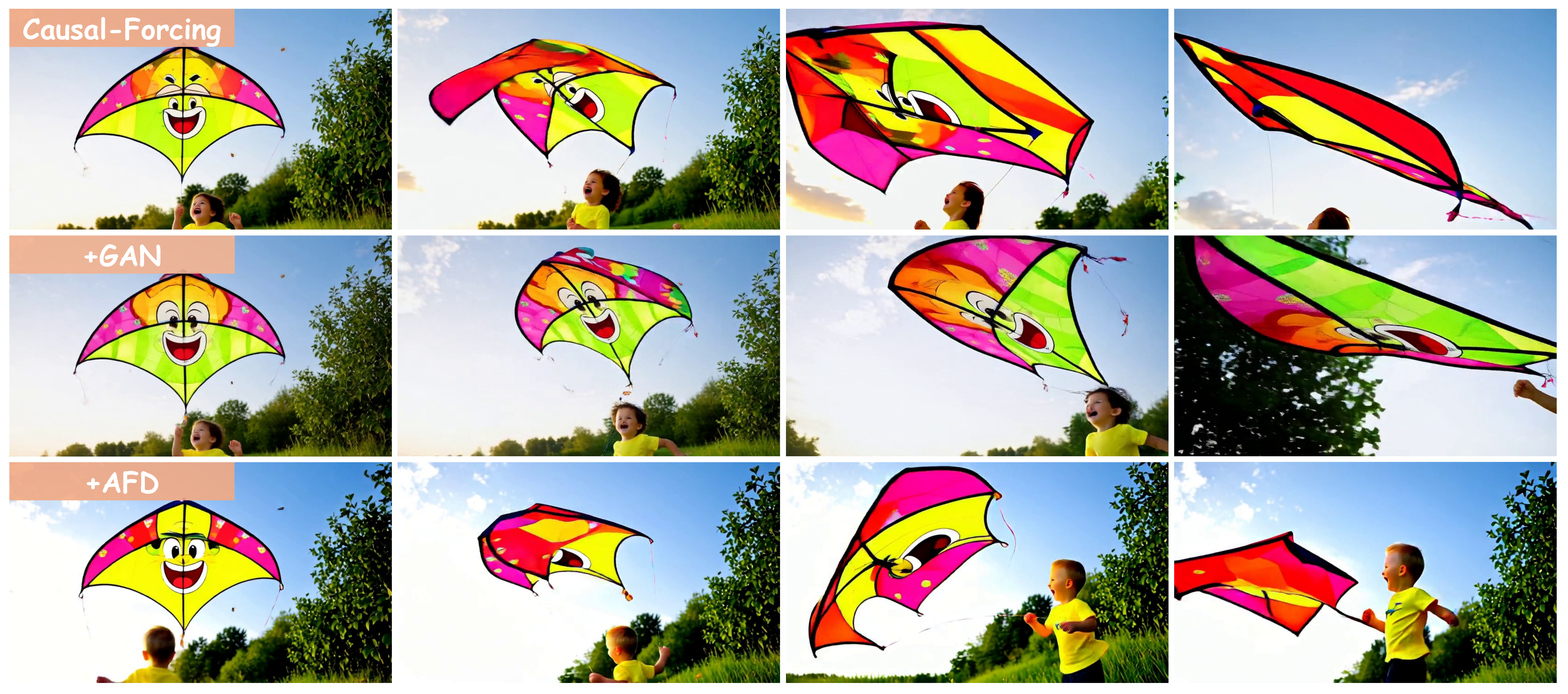}
\end{minipage}
\caption{Qualitative visualizations of AFD in the main experiment. The first row shows Self-Forcing examples, and the second row shows Causal-Forcing examples. Each example shows representative teacher--student comparison frames under the same prompt, illustrating how the on-policy AFD update changes the student's generated video distribution. Prompts are given in Table~\ref{tab:prompts_main_qualitative}.}
\label{fig:afd_qualitative}
\vspace{-1.5em}
\end{figure}

Figure~\ref{fig:afd_qualitative} supports the trends in Tables~\ref{tab:main_results} and~\ref{tab:motion_physics_aux}. Across both AR student families, AFD keeps the main prompt content and visual quality while improving motion-related behavior. This matches the method design: the discriminator scores completed on-policy rollouts, and the DiffusionNFT update transfers this signal to the student's forward-noised states. The examples therefore support the quantitative gains in dynamic degree, motion quality, and physical consistency. Additional examples are provided in Appendix~\ref{sec:additional_qualitative}.

\subsection{Ablations}
\label{sec:ablations}

\textbf{Continual learning on a stylized domain.}
We further test AFD on a shifted visual domain using 200 Disney-style animation prompts. Figure~\ref{fig:afd_anime_qualitative} shows that the adapted student follows the new style while preserving prompt alignment and coherent motion. This suggests that AFD is not limited to the physics prompt distribution used in the main experiment. It supports continual learning by using the teacher--student discriminator to provide feedback on the current target domain and the forward-process update to transfer this feedback to the student's denoising states, without teacher scores or architectural alignment.

\begin{figure}[t]
\centering
\begin{minipage}[t]{0.49\textwidth}
\centering
\includegraphics[width=\linewidth]{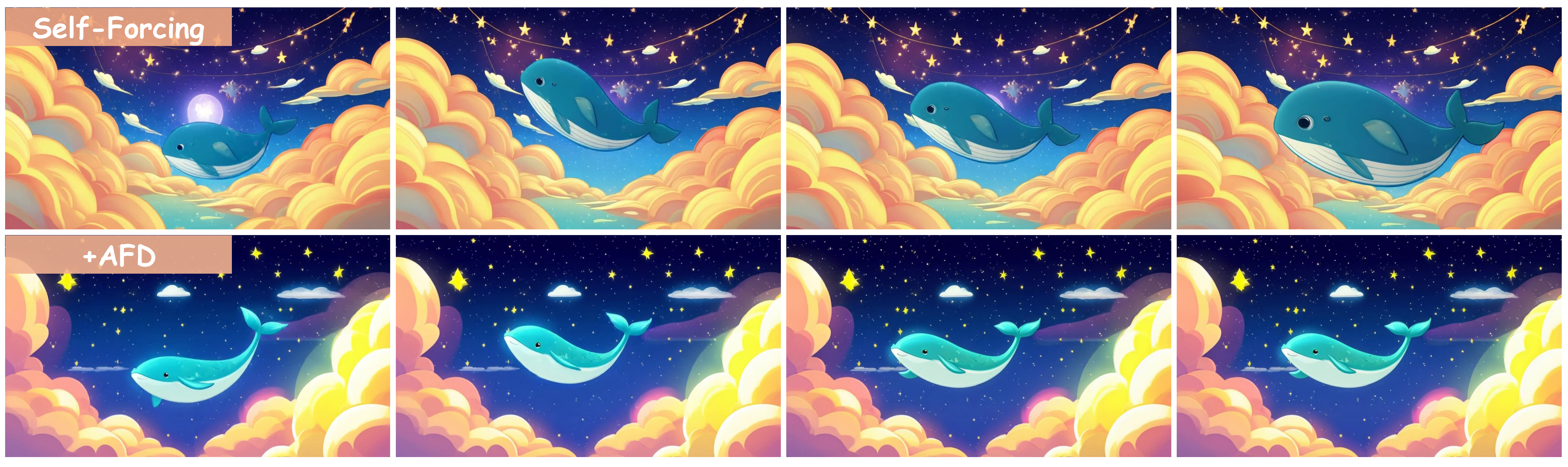}
\end{minipage}
\hfill
\begin{minipage}[t]{0.49\textwidth}
\centering
\includegraphics[width=\linewidth]{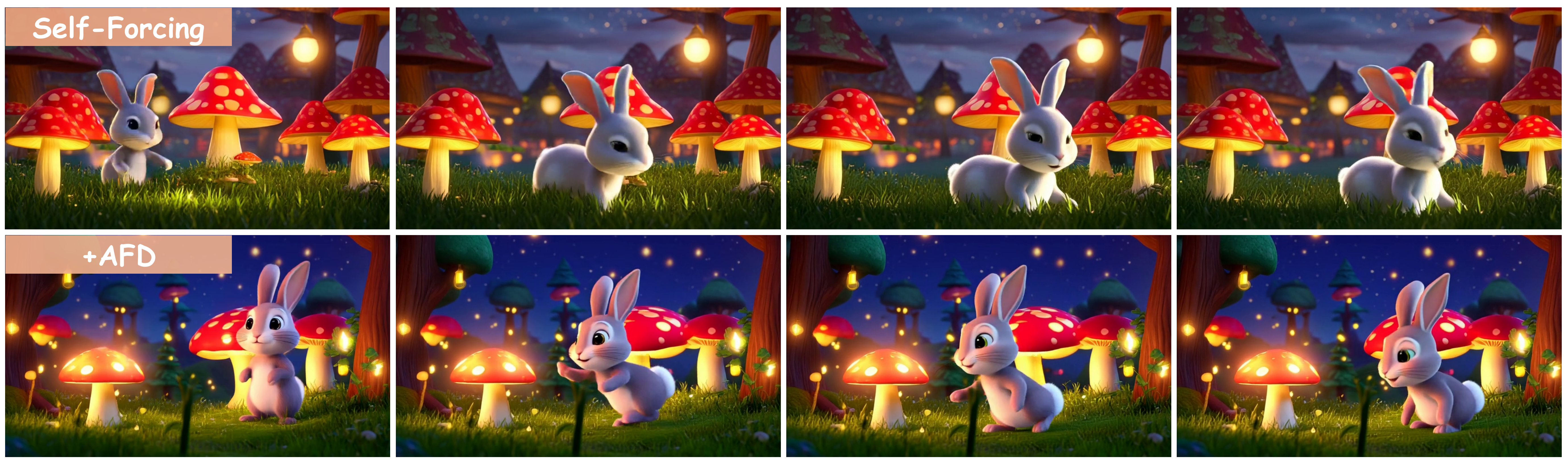}
\end{minipage}
\caption{Qualitative examples of AFD on a Disney-style animation dataset for continual learning. Detailed prompts are provided in Table~\ref{tab:prompts_anime_qualitative}.}
\label{fig:afd_anime_qualitative}
\vspace{-20pt}
\end{figure}

\textbf{Discriminator learning rate.}
Because AFD's supervision is generated by a co-evolving discriminator, the discriminator update rate $\eta_D$ directly controls how informative the student's reward signal is over training. We sweep $\eta_D\in\{0,\,1\!\times\!10^{-6},\,5\!\times\!10^{-6},\,1\!\times\!10^{-5},\,5\!\times\!10^{-5}\}$ with all other hyperparameters fixed, and track the discriminator reward $r_\phi$ on student rollouts.

Figure~\ref{fig:disc_lr_ablation} shows that both under-updating and over-updating the discriminator degrade the quality of the rollout reward. When $\eta_D\!\le\!1\!\times\!10^{-6}$, the discriminator lags behind the evolving student distribution: $r_\phi$ rapidly approaches $1$, indicating saturation of a stale scoring function rather than reliable reduction of the teacher--student gap. When $\eta_D=5\!\times\!10^{-5}$, the discriminator becomes overly strong early in training; $r_\phi$ remains near $0$ across the batch, reducing the contrast between positive and negative rollout weights and weakening the student update. Intermediate rates ($5\!\times\!10^{-6}$ and $1\!\times\!10^{-5}$) avoid these saturation regimes and produce a gradual increase in $r_\phi$, which is consistent with the intended role of $D_\phi$ as an adaptive estimator of the current teacher--student discrepancy rather than a fixed reward model.

\begin{wrapfigure}{r}{0.7\textwidth}
\centering
\vspace{-1.5em}
\includegraphics[width=0.68\textwidth]{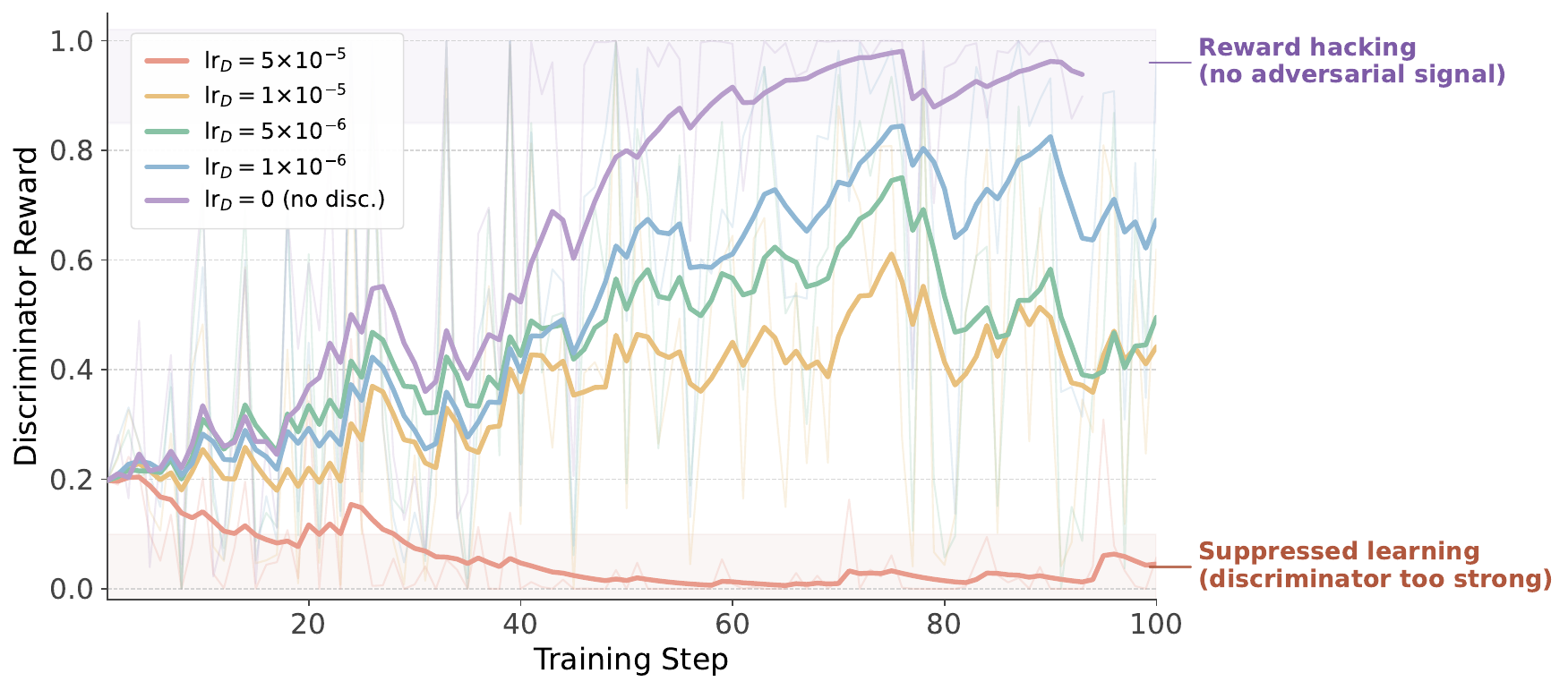}
\vspace{-0.5em}
\caption{Discriminator learning-rate ablation. The rollout reward $r_\phi$ is most informative at intermediate $\eta_D$; overly small or large update rates lead to reward hacking or suppressed learning.}
\label{fig:disc_lr_ablation}
\vspace{-1.0em}
\end{wrapfigure}

\textbf{Discriminator loss.}
We next replace the Bradley--Terry discriminator objective from Section~\ref{sec:method} with a GAN-style binary classification loss while keeping the same on-policy data, discriminator architecture, and forward-process student update. Figure~\ref{fig:vbench_bt_vs_gan} compares the resulting VBench motion and physics dimensions. BT improves dynamic degree by $9.72$ points over the GAN loss ($88.89$ vs.\ $79.17$) and slightly improves motion smoothness ($98.89$ vs.\ $98.53$), while temporal flickering, human action, and spatial relation remain effectively tied.

\begin{wrapfigure}{r}{0.50\textwidth}
\centering
\vspace{-1.0em}
\includegraphics[width=0.48\textwidth]{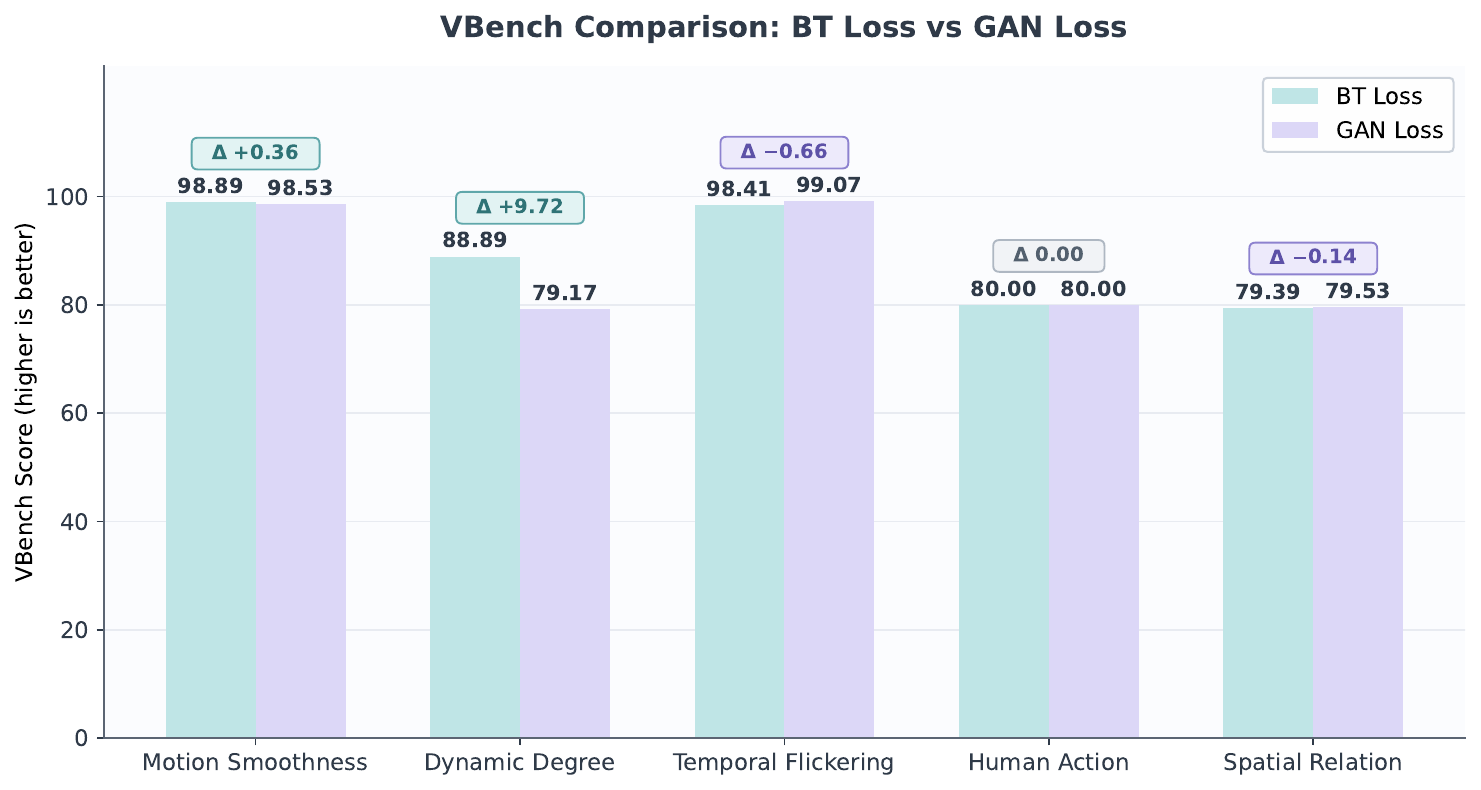}
\caption{VBench ablation of the discriminator loss. Replacing the prompt-paired Bradley--Terry objective with a GAN-style binary loss mainly reduces dynamic degree, while the other motion and physics dimensions remain close.}
\label{fig:vbench_bt_vs_gan}
\vspace{-1.0em}
\end{wrapfigure}

This pattern matches the role of the discriminator in AFD. The student update does not need a globally calibrated real/fake probability; it needs a reliable on-policy ranking signal for forming positive and negative rollouts in the forward-process objective. The BT loss compares teacher and student videos under the same prompt and optimizes the margin $D_\phi(x_0^T,y)-D_\phi(\hat{x}_0,y)$, so prompt difficulty, content category, and static appearance biases are largely cancelled inside each pair. A GAN-style loss instead trains an absolute classifier over teacher and student samples, which can be dominated by easy distributional cues and can saturate when the discriminator becomes too confident. BT therefore provides a smoother relative advantage for AFD, yielding stronger motion improvement while keeping the general video generation capability of the student.

\section{Conclusion}
AFD suggests that black-box video distillation should be formulated around the information interface available to the student, rather than around the teacher's hidden generation process. For causal AR students, the missing ingredient is not access to the teacher trajectory, but a principled way to connect observable sample-level discrepancies to the student's own denoising-time states. By separating these two roles into adaptive discrimination and forward-process credit assignment, AFD turns completed-video supervision into an on-policy training signal that improves motion and physical behavior without sacrificing general quality. This perspective offers a practical and extensible route for transferring capabilities from powerful sampling-only video models to efficient autoregressive generators.

\section{Limitations}
Our experiments focus on two causal autoregressive student backbones and a small set of target domains, mainly physics-oriented prompts and stylized animation prompts. Broader studies across more teachers, longer videos, and more diverse prompt distributions would further test the generality of AFD. In addition, AFD uses an online discriminator, so its performance can depend on discriminator update rate and reward calibration; our ablations provide initial guidance, but a more systematic tuning study is left for future work. Evaluation on broader reasoning-oriented benchmarks such as V-ReasonBench \citep{luo2025vreasonbenchunifiedreasoningbenchmark} is left for future work.

\bibliography{references}

\section{Hyperparameters}
\label{sec:hyperparameters}

We provide detailed hyperparameters in Table~\ref{tab:hyperparameters}. The table lists the model, optimization, diffusion, regularization, and rollout settings used for AFD training.

\begin{table}[h]
\centering
\caption{Comprehensive hyperparameters for AFD training. We detail the configurations used across model and video specifications, LoRA fine-tuning, optimization, diffusion process, selective regularization, and streaming rollout.}
\label{tab:hyperparameters}
\small
\setlength{\tabcolsep}{4pt}
\renewcommand{\arraystretch}{1.08}
\begin{tabular}{p{0.26\linewidth}p{0.34\linewidth}p{0.30\linewidth}}
\toprule
\textbf{Module} & \textbf{Hyperparameter} & \textbf{Value} \\
\midrule
\multirow{2}{*}{\textbf{Model \& Video Specs}}
& Base Architecture & Self-Forcing / Causal-Forcing \\
& Video Resolution ($H \times W$) & $480 \times 832$ \\
\midrule
\multirow{4}{*}{\textbf{LoRA Fine-Tuning}}
& Rank ($r$) & 256 \\
& Scaling Factor ($\alpha$) & 256 \\
& Dropout Rate & 0.0 \\
& Gradient Checkpointing & Enabled \\
\midrule
\multirow{7}{*}{\textbf{Optimization}}
& Precision Mode & bfloat16 \\
& Optimizer & AdamW ($\beta_1 = 0.9, \beta_2 = 0.999, \epsilon = 1e-8$) \\
& Learning Rate ($\eta$) & 1e-5 \\
& Weight Decay & 1e-4 \\
& Max Gradient Norm & 1.0 \\
\midrule
\multirow{6}{*}{\textbf{Selective Regularization}}
& Interpolation Strength ($\beta$) & 0.1 \\
& Prior Loss Weight ($\lambda_{\mathrm{prior}}$) & 1e-4 \\
& Advantage Clip Max & 5.0 \\
& Reward Normalization & Global Std with Per-Prompt Tracking \\
& EMA Decay Rate ($\gamma$) & 0.99 \\
\bottomrule
\end{tabular}
\end{table}



\section{Additional Qualitative Visualizations}
\label{sec:additional_qualitative}

We provide additional AFD visualizations in img~\ref{fig:afd_appendix_demo_2}--\ref{fig:afd_appendix_demo_7}. These examples further illustrate the same pattern observed in the main experiments: AFD improves motion evolution and physical plausibility while preserving the prompt-level semantics and visual fidelity of the base autoregressive video student.

\begin{figure}[p]
\centering
\includegraphics[width=\textwidth]{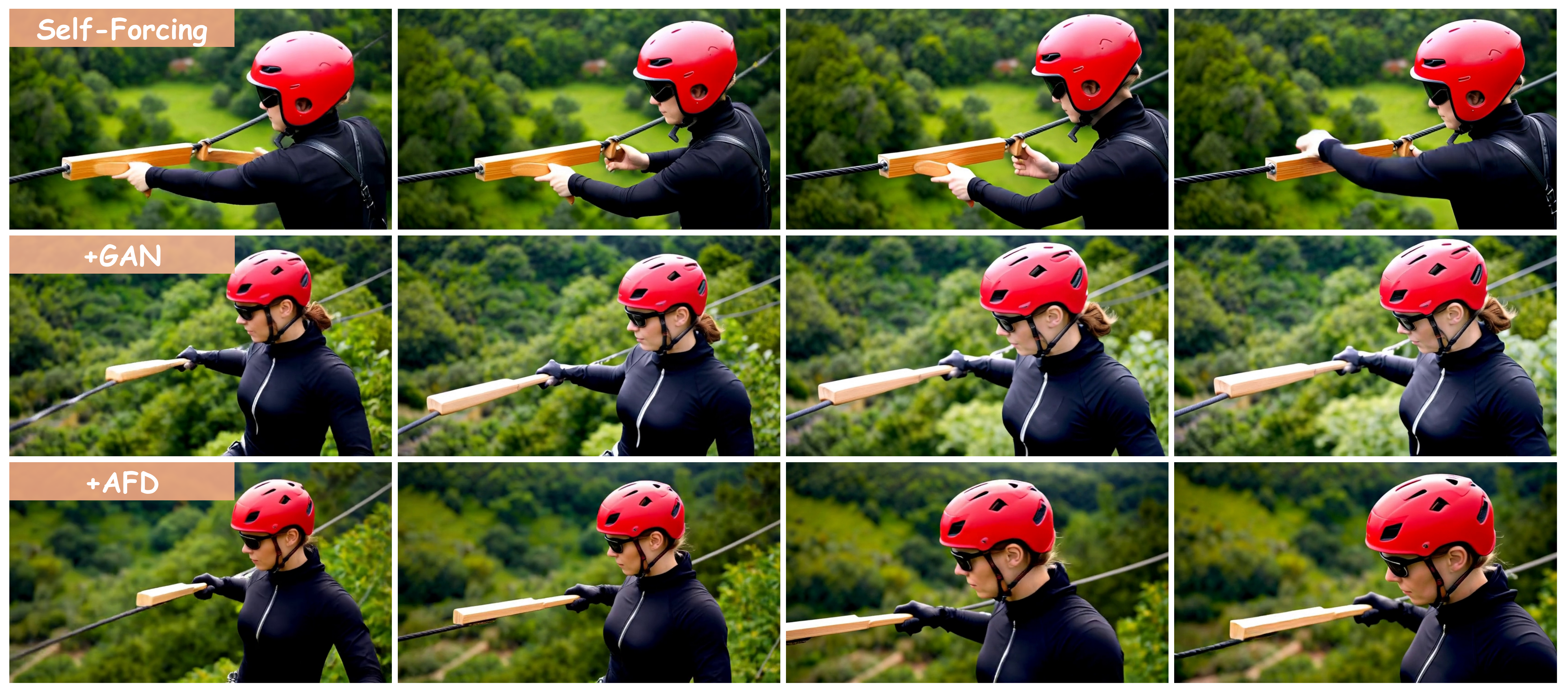}
\caption{Additional qualitative visualization of AFD. Prompt: A breathtaking tightrope walker, clad in a sleek black outfit and a striking red helmet, navigates a narrow tightrope suspended high above a lush green landscape. The camera captures a close-up of their focused face, framed by a dark visor, as they expertly manipulate a wooden block along the rope using a long, slender balance pole. The block glides smoothly, creating a subtle groove in the taut rope, a testament to their precision and control. The background blurs into a vibrant tapestry of greenery, enhancing the sense of isolation and danger. The tightrope, a thin line of tension, contrasts sharply with the serene environment, while the camera remains steady, allowing the viewer to fully immerse in the walker's intense concentration and the delicate dance of balance and skill.}
\label{fig:afd_appendix_demo_2}
\end{figure}

\begin{figure}[p]
\centering
\includegraphics[width=\textwidth]{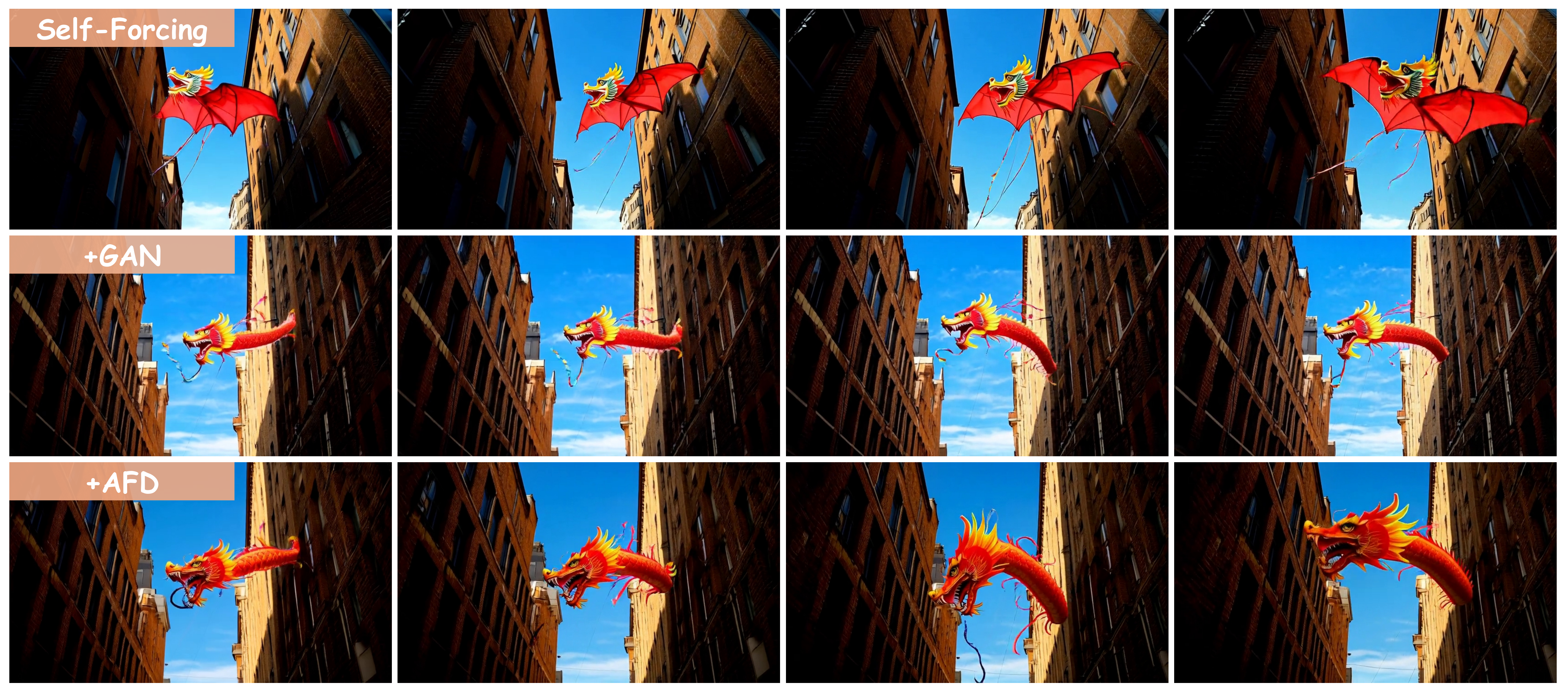}
\caption{Additional qualitative visualization of AFD. Prompt: A majestic dragon kite soars through a narrow urban canyon, framed by towering, weathered brick buildings. The kite, intricately crafted with a vibrant red body and a fierce, golden head, glides effortlessly against a backdrop of deep blue skies, its tail adorned with a cascade of colorful ribbons that flutter in the breeze. The camera captures the kite's dynamic flight, showcasing its graceful turns and the subtle interplay of light and shadow as it navigates the tight space between the buildings. The scene is enhanced by a soft, golden-hour glow, casting a warm hue over the textured brick facades, while the kite's movement creates a striking contrast against the static architecture. This cinematic moment celebrates the artistry of kite flying, blending the beauty of nature with the urban landscape.}
\label{fig:afd_appendix_demo_4}
\end{figure}

\begin{figure}[p]
\centering
\includegraphics[width=\textwidth]{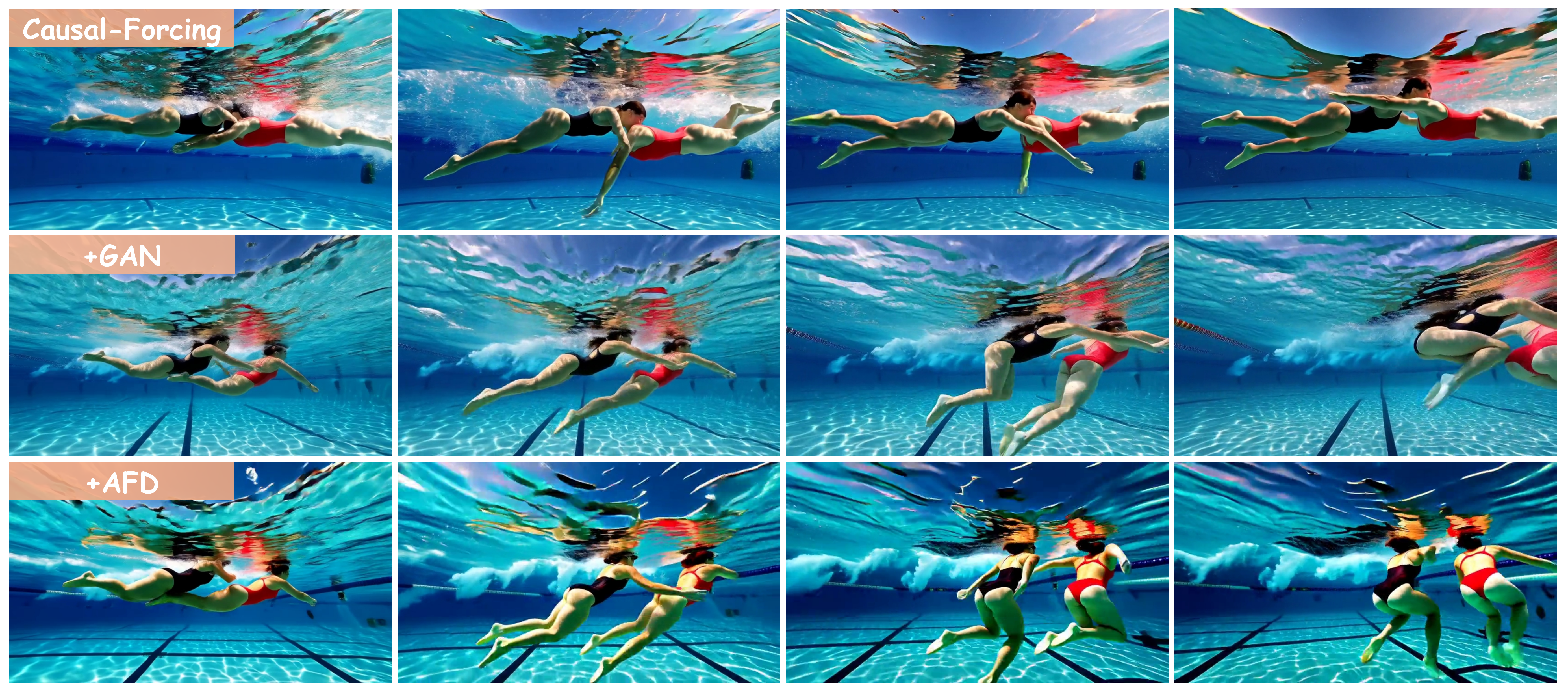}
\caption{Additional qualitative visualization of AFD. Prompt: In a breathtaking underwater tableau, two swimmers glide through the crystal-clear water, their synchronized breaststroke movements creating a mesmerizing dance of rhythm and grace. The camera, positioned at a slight angle, captures the swimmers' powerful legs and arms as they propel themselves forward, their bodies slicing through the water with precision. The first swimmer, clad in a sleek black swimsuit, leads the pack, while the second, in a vibrant red suit, trails closely, their synchronized movements echoing a fierce competition. The water's surface shimmers with sunlight, casting dynamic reflections that enhance the scene's ethereal quality. The static camera allows for a focused view of their synchronized strokes, highlighting the swimmers' strength and agility, while the serene underwater environment amplifies the beauty of their synchronized performance.}
\label{fig:afd_appendix_demo_7}
\end{figure}

\section{Prompts for Qualitative Visualizations}
\label{sec:visualization_prompts}

We present the prompts used for the qualitative examples in img~\ref{fig:afd_qualitative} and~\ref{fig:afd_anime_qualitative}.

\begin{table}[h]
\centering
\caption{Prompts used for Figure~\ref{fig:afd_qualitative}.}
\label{tab:prompts_main_qualitative}
\small
\setlength{\tabcolsep}{4pt}
\renewcommand{\arraystretch}{1.12}
\begin{tabular}{p{0.24\linewidth}p{0.70\linewidth}}
\toprule
\textbf{Example} & \textbf{Prompt} \\
\midrule
physics1 & In a breathtaking display of aquatic athleticism, a skilled skier glides effortlessly across the shimmering surface of a tranquil lake, propelled by a powerful ski jet. The scene opens with the skier, clad in a vibrant red and black wetsuit, gracefully poised on the water, their skis slicing through the glassy surface. The ski jet, a sleek black vessel with a striking red stripe, glides alongside, its engine humming with precision as it maintains a steady pace. The camera captures the dynamic interplay between the skier and the jet, showcasing the skier's agility and balance as they lean into the turns, their movements fluid and controlled. The backdrop features a serene landscape of lush greenery and distant mountains, enhancing the tranquil ambiance. As the skier gains momentum, the camera remains fixed, allowing viewers to appreciate the exhilarating dance of speed and skill on the water, culminating in a thrilling display of aquatic prowess. \\
physics2 & In a serene bathroom setting, a person clad in a light blue shirt and dark pants stands before a sleek white sink, the soft glow of natural light filtering through a nearby window. The camera captures a close-up of their hands, one gently squeezing a vibrant blue toothpaste tube, while the other holds a pristine white toothbrush, poised to receive the creamy substance. As the tube is squeezed, a steady stream of toothpaste flows onto the bristles, creating a mesmerizing cascade of blue that contrasts beautifully with the white brush. The scene is bathed in a warm, inviting light, enhancing the textures of the toothpaste and the brush, while the background remains softly blurred, drawing the viewer's focus to the meticulous act of oral hygiene. The static camera angle ensures a clear view of the process, capturing the rhythmic motion of the hands and the satisfying result of a perfectly coated toothbrush. \\
physics3 & In a serene, softly lit room, a pair of hands gracefully manipulates knitting needles, their rhythmic movements creating a mesmerizing dance of craftsmanship. The camera captures a close-up of the hands, framed against a blurred backdrop, enhancing the focus on the intricate process. The needles glisten under the warm, diffused light, casting gentle shadows that accentuate the texture of the yarn. As the hands deftly weave the yarn, the knitted piece grows visibly longer, each stitch a testament to the skill and patience of the unseen artisan. The static camera angle allows for an intimate exploration of the technique, while the subtle color grading enhances the cozy atmosphere, inviting viewers to immerse themselves in the artistry of knitting. \\
physics4 & A vibrant kite, adorned with whimsical cartoon characters, soars against a backdrop of lush greenery and a clear blue sky, its colorful design capturing the essence of childhood joy. The kite, featuring a playful character with a wide smile and a distinctive hat, is held by a young child in a bright yellow shirt, who eagerly runs in a circular motion, their laughter echoing through the air. The camera, positioned at a low angle, captures the dynamic interplay between the child's enthusiasm and the kite's graceful ascent, as it dances on the breeze. The scene is bathed in soft, natural light, enhancing the vivid hues of the kite and the child's clothing, while the gentle rustle of leaves and the distant hum of nature create an immersive atmosphere. This moment encapsulates the simple pleasures of outdoor play, inviting viewers to share in the exhilaration of flight and the boundless energy of youth.r \\
\bottomrule
\end{tabular}
\end{table}

\begin{table}[h]
\centering
\caption{Prompts used for Figure~\ref{fig:afd_anime_qualitative}.}
\label{tab:prompts_anime_qualitative}
\small
\setlength{\tabcolsep}{4pt}
\renewcommand{\arraystretch}{1.12}
\begin{tabular}{p{0.24\linewidth}p{0.70\linewidth}}
\toprule
\textbf{Example} & \textbf{Prompt} \\
\midrule
anime1 & A brave cartoon rabbit exploring a glowing mushroom forest at dusk, expressive 3D cartoon character, magical storybook atmosphere, warm cinematic lighting, whimsical fantasy environment, soft polished textures, original fairy-tale-inspired animation style. \\
anime2 & A cartoon whale floating through a dreamy sky filled with stars and clouds, whimsical fantasy sequence, polished 3D cartoon family-animation look, cinematic soft glow, original fairy-tale-inspired animation style. \\
\bottomrule
\end{tabular}
\end{table}

\section{Theoretical Insights}
\label{sec:appendix_theory}
In this section, we provide theoretical analysis on the proposed method Adversarial Flow Distillation (AFD). We reveal the connection between policy optimization in AFD to the KL-divergence on-policy distillation from teacher.

Let $\pi_T(x_0 | y)$ denote the teacher distribution accessible only through clean samples, and $\pi_\theta(x_0 | y)$ the student parameterized by velocity field $v_\theta(x_t, y, t)$. Let $v^{\text{old}}(x_t, y, t)$ denote the EMA data-collection policy, with induced distribution $\pi^{\text{old}}$. The forward process is
\begin{equation}
x_t = \alpha_t x_0 + \sigma_t \epsilon, \qquad \epsilon \sim \mathcal{N}(0, I),
\end{equation}
with sample-level forward velocity
\begin{equation}
v(x_0, \epsilon, t) = \dot\alpha_t x_0 + \dot\sigma_t \epsilon.
\end{equation}

The discriminator $D_\phi(\hat{x_0}, y) \in (0,1)$ is trained to distinguish teacher samples $x_0^T \sim \pi_T$ from student rollouts $\hat x_0 \sim \pi_\theta$, with optimum
\begin{equation}
D^*_\phi(\hat{x_0}, y) = \frac{\pi_T(\hat{x_0} | y)}{\pi_T(\hat{x_0} | y) + \pi_\theta(\hat{x_0} | y)}, \qquad \rho^*_\phi(\hat{x_0}, y) := \log\frac{D^*_\phi(\hat{x_0}, y)}{1 - D^*_\phi(\hat{x_0}, y)} = \log\frac{\pi_T(\hat{x_0}| y)}{\pi_\theta(\hat{x_0}| y)}.
\end{equation}

In the AFD setting, if we set advantage-form reward dependent on the density ratio estimated by discriminator:
\begin{equation}
r_\phi(x_0, y) := \exp \big(  \rho_\phi(x_0, y) \big).
\end{equation}

Recall from DiffusionNFT Section~3.1:
\begin{equation}
\pi^+(x_0 | c) := \pi^{\text{old}}(x_0 | o = 1, c) = \frac{r(x_0, c)}{p_{\pi^{\text{old}}}(o=1 | c)}\, \pi^{\text{old}}(x_0 | c).
\end{equation}
This is the conditional distribution of $x_0$ given that the optimality variable $o$ equals $1$ --- i.e., the reward-reweighted version of $\pi^{\text{old}}$ that puts more mass on high-reward samples. The tilted policy can also be written as
\begin{equation}
\pi^+(x_0 | y) \propto \exp \big(  \rho_\phi  \big) \cdot \pi^{\text{old}}(x_0 | y),
\end{equation}
and at optimal $\phi^*$ due to the on-policy setting $\pi^{\text{old}} = \pi_\theta$, this reduces to
\begin{equation}
\pi^+(x_0 | y) \propto \pi_T(x_0).
\label{eq:pi_optimal}
\end{equation}
Equation \ref{eq:pi_optimal} demonstrates that the tilted policy $\pi^+$ under the optimal discriminator recovers the teacher policy. In the following proposition, we show that the policy optimization process in AFD pushes the policy towards the teacher model.
\begin{proposition}
Under an optimal discriminator, the policy improvement step from
$\pi^{\mathrm{old}}$ to the reward-tilted distribution $\pi^+$ is
equivalent to on-policy reverse-KL distillation from the teacher $\pi_T$.
\end{proposition}
\begin{proof}
The tilted policy $\pi^+$ is the solution to the optimization problem:
\begin{align}
\max_{\pi}\; \mathbb{E}_{x_0 \sim \pi}[\rho(x_0, y)] - \mathrm{KL}\!\left(\pi \,\|\, \pi^{\mathrm{old}}\right).
\end{align}
Plugging the optimal discriminator into $\rho$, we obtain that the optimization problem becomes
\begin{align}
&\max_{\pi}\; \mathbb{E}_{x_0 \sim \pi}\!\left[
\log \frac{\pi_T(x_0)}{\pi_\theta(x_0)}
\right]
-
\mathrm{KL}\! \left(\pi \,\|\, \pi^{\mathrm{old}}\right) \\
\Longleftrightarrow &\max_{\pi}\; -\mathrm{KL}\!\left(\pi \,\|\, \pi_T\right) + \mathrm{KL}\!\left(\pi \,\|\, \pi_\theta \right)
-
\mathrm{KL}\!\left(\pi \,\|\, \pi^{\mathrm{old}}\right) \\
\Longleftrightarrow &\max_{\pi}\; -\mathrm{KL}\!\left(\pi \,\|\, \pi_T\right)
\label{eq:todistill}
\end{align}
Equation \ref{eq:todistill} finally becomes reverse-KL distillation from the teacher model.
\end{proof}

\end{document}